\documentclass[twoside,11pt]{article}

\usepackage{jmlr2e}

\firstpageno{1}

\usepackage{amsmath}
\usepackage{amssymb}
\usepackage{mathtools}
\usepackage{amsthm}
\usepackage{lipsum}

\usepackage{tikz}
\usetikzlibrary{positioning, arrows.meta}

\usepackage[capitalize,noabbrev]{cleveref}
\usepackage{algorithm,algorithmic}

\newtheorem{fact}{Fact}

\usepackage[textsize=tiny]{todonotes}

\DeclarePairedDelimiter{\ceil}{\lceil}{\rceil}

\DeclareMathOperator*{\argmax}{arg\,max} 
 
\newcommand{\cjq}[1]{\color{blue}#1\color{black}}

\newcommand{\ngy}[1]{\color{cyan}#1\color{black}}
\newcommand{\etckgk}{C-ETC-K}
\newcommand{\etckga}{C-ETC-Y}

\makeatletter
\newcommand\blfootnote[1]{%
  \begingroup
  \renewcommand{\@makefntext}[1]{\noindent\makebox[1.8em][r]#1}
  \renewcommand\thefootnote{}\footnote{#1}%
  \addtocounter{footnote}{-1}%
  \endgroup
}
\makeatother

\newcommand{\cmin}{c_\mathrm{min}}
\newcommand{\mopt}{m^*}
\newcommand{\moptrnd}{m^\dagger}

\usepackage{comment}
\usepackage{natbib}
\usepackage{subfig}

\includecomment{commentediting}
\excludecomment{commentediting}

\includecomment{commentediting2}
\excludecomment{commentediting2}

\includecomment{commentediting3}
\excludecomment{commentediting3}

\begin{document}

\title{A Framework for Adapting Offline Algorithms to Solve Combinatorial Multi-Armed Bandit Problems with Bandit Feedback}

\author{\name Guanyu Nie\email nieg@iastate.edu \\
	\name Yididiya Y Nadew\email yididiya@iastate.edu \\
	\name Yanhui Zhu\email yanhui@iastate.edu \\
	\name Vaneet Aggarwal\email{vaneet@purdue.edu}\\
	\name Christopher John Quinn\email{cjquinn@iastate.edu}
}
%

\maketitle

\begin{abstract}
  We investigate the problem of stochastic, combinatorial multi-armed bandits where the learner only has access to bandit feedback and the reward function can be non-linear. We provide a general framework for adapting discrete offline approximation algorithms into sublinear $\alpha$-regret methods that only require bandit feedback, %
  achieving $\mathcal{O}\left(T^\frac{2}{3}\log(T)^\frac{1}{3}\right)$ expected cumulative $\alpha$-regret dependence on the horizon $T$.  The framework only requires the offline algorithms to be robust to small errors in function evaluation.  The adaptation procedure does not even require explicit knowledge of the offline approximation algorithm --- the offline algorithm can be used as a black box subroutine. To demonstrate the utility of the proposed framework, the proposed framework is applied to diverse applications in submodular maximization. The new CMAB algorithms for submodular maximization with knapsack constraints outperform a full-bandit method developed for the adversarial setting in experiments with real-world data.

\end{abstract}

\section{Introduction}
\label{intro}

Many real world sequential decision problems can be modeled using the framework of stochastic multi-armed bandits (MAB), such as scheduling,  assignment problems, advertising campaigns, and product recommendations, among others.  The decision maker sequentially selects actions and receives stochastic rewards from an unknown distribution.  The goal of the decision maker is to maximize the expected cumulative reward over a (possibly unknown) time horizon.  Actions result both in the immediate reward and, more importantly, information about that action's reward distribution. Such problems result in a trade-off between trying actions the agent is uncertain of (\textit{exploring}) or only taking the action that is empirically the best seen so far (\textit{exploiting}).\blfootnote{This extends the framework from \citep{nie23framework} to adapt randomized offline approximation algorithms.}

In the classic MAB setting, the number of possible actions is small relative to the time horizon, meaning each action can be taken at least once, and %
there is no assumed relationship between the reward distributions of different arms. %
The \textit{combinatorial} multi-armed bandit (CMAB) setting involves a large but structured action space.  For example, in product recommendation problems, the decision maker may select a subset of products (base arms) from among a large set. There are several aspects that can affect the difficulty of these problems.   First, MAB methods are typically compared against a learner with access to a value oracle of the reward function (an offline problem).  For some problems, it is NP-hard for the baseline learner with value oracle access to optimize.  An example is if the expected/averaged reward function is submodular and actions are subsets constrained by cardinality.  At best, for these problems, approximation algorithms may exist.  Thus, unless the time horizon is large (exponentially long in the number of base arms, for instance), it would be more reasonable to compare the CMAB agent against the performance of the  approximation algorithm for the related offline problem. Likewise, one could apply state of the art methods for (unstructured) MAB problems treating each subset as a separate arm, and obtain $\tilde{\mathcal{O}}(T^\frac{1}{2})$ dependence on the horizon $T$ for the subsequent regret bound.  However, that dependence would only apply for exponentially large $T$.

Feedback plays an important role in how challenging the problem is.  When the decision maker only observes a (numerical) reward for the action taken, that is known as bandit or full-bandit feedback.  When the decision maker observes additional information, such as contributions of each base arm in the action, that is semi-bandit feedback. Semi-bandit feedback greatly facilitates learning for CMAB problems.   Suppose for example that the reward function (on average) was non-linear but monotone increasing over the inclusion lattice of the $n$ base arms and there was a cardinality constraint of size $k$ characterizing the action set. 
 The agent would know from the start that no set of size smaller than $k$ could be optimal. %
 However, there would be $n \choose k$ sets of size  $k$.  For $n=100$ and $k=10$, the agent would need a horizon $T>10^{12}$ to try each cardinality $k$ set even just once.  If the reward function belongs to a certain class, such as the class of submodular functions, then one approach would be to use a greedy procedure based on base arm values.   With semi-bandit feedback, the agent could, on the one hand,  take actions of cardinality $k$, 
 gain the subsequent rewards, and yet also observe samples of the base arms' values to improve future actions.

Bandit feedback is much more challenging, as only the joint reward is observed.  In general, for non-linear reward functions,  the individual values or marginal gains of base arms can only be loosely bounded if actions only consist of maximal subsets. Thus, to estimate values or marginal gains of base arms, the agent would need to deliberately spend time  sampling actions  that are known to be sub-optimal (small sets of base arms) in order to estimate their values to later  better select actions of cardinality $k$.  Standard MAB methods like UCB or TS based methods by design do not take actions  identified (with some confidence)  to be sub-optimal.  Thus, while such strategies can be employed when semi-bandit feedback is available, it is less clear whether they can be effectively used when only bandit feedback is available. %

There are important applications where semi-bandit feedback may not be available, such as in influence maximization and recommender systems.  Influence maximization models the problem of identifying a low-cost subset (seed set) of nodes in a (known) social network that can influence the maximum number of nodes in a network \citep{Nguyen2013OnBI, Leskovec2007CosteffectiveOD, Bian2020Efficient}. %
Recent research has generalized the problem to online settings where the knowledge of the network and diffusion model is not required \citep{Wang2020FastTS, perrault2020budgeted} but extra feedback is assumed.  However, for many networks the user interactions and user accounts are private; only aggregate feedback (such as the count of individuals using a coupon code or going to a website) might be visible to the decision maker.

In this work, we seek to address these challenges by proposing a {\em general framework} for adapting offline approximation algorithms into algorithms for stochastic CMAB problems when only bandit feedback is available.   We identify that a single  condition related to the robustness of the approximation algorithm to erroneous function evaluations is sufficient to guarantee that a simple explore-then-commit (ETC) procedure accessing the approximation algorithm as a black box results in a sublinear $\alpha$-regret CMAB algorithm despite having only bandit feedback available.  The approximation algorithm does not need to have any special structure (such as an iterative greedy design).  Importantly, no effort is needed on behalf of the user in mapping steps in the offline method into steps of the CMAB method.

We demonstrate the utility of this framework by assessing the robustness of several approximation algorithms. We consider submodular maximization problems where we study three approximation algorithms designed for knapsack constraints, two designed for cardinality constraints and one designed for unconstrained problems, which immediately result in sublinear $\alpha$-regret CMAB algorithms that only rely on bandit-feedback. We note that this paper provides first regret guarantees for stochastic CMAB problems with submodular rewards, knapsack constraints, and bandit feedback.  We also show that despite the simplicity and universal design of the adaptation, the resulting CMAB algorithms work well on budgeted influence maximization and song recommendation problems using real-world data. 

The main contributions of this paper can be summarized as: 

\noindent {\bf 1.} We provide a general framework for adapting discrete offline approximation algorithms into sublinear $\alpha$-regret methods for stochastic CMAB problems where only bandit feedback is available.  The framework only requires the offline algorithms to be robust to small errors in function evaluation, a property important in its own right for offline problems.  The algorithms are not required to have a special structure --- instead they are used as black boxes.  Our procedure has minimal storage and time-complexity overhead and achieves a regret bound with $\tilde{\mathcal{O}}(T^\frac{2}{3})$ dependence on the horizon $T$.

\noindent {\bf 2.} We illustrate the utility of the proposed framework by assessing the robustness of several approximation algorithms for (offline) constrained submodular optimization, a class of reward functions lacking simplifying properties of linear or Lipschitz reward functions.  Specifically, we prove the robustness of approximation algorithms given in \citep{nemhauser1978analysis, badanidiyuru2014fast, Sviridenko2004ANO, khuller1999budgeted, yaroslavtsev2020bring} with cardinality or knapsack constraints, and use the general framework to give regret bounds for the stochastic CMAB. In particular, we note that this paper gives the first regret bounds for stochastic submodular CMAB with knapsack constraints under bandit feedback. 

\noindent {\bf 3.} We empirically evaluate the performance of proposed framework through the stochastic submodular CMAB with knapsack constraints problem for two applications: Budgeted Influence Maximization and Song Recommendation. The evaluation results demonstrate that the proposed approach significantly outperforms a full-bandit method for a related problem in the adversarial setting.

The rest of the paper is organized as follows. Section \ref{sec:related-work} describes the key related works. Section \ref{prob_state} defines the considered combinatorial multi-armed bandit (CMAB) problem. Section \ref{sec:robust} defines the robustness guarantee of offline combinatorial optimization problem. Section \ref{sec:alg} provides the proposed algorithm using which  the offline approximation algorithm could be adapted to the stochastic CMAB algorithm, and provides the regret guarantees for stochastic CMAB. Section \ref{sec:appl-CETC:submod} applies the proposed framework to different submodular maximization problems. Section \ref{sec:exp} provides the evaluation results of the proposed framework on stochastic submodular CMAB with knapsack constraints problem. Section \ref{sec:concl} concludes the paper.

\section{Related Work} \label{sec:related-work}

\begin{table}[t]
\begin{tabular}{c|c|c|c}
\hline
Application                                                                        & \begin{tabular}[c]{@{}c@{}}Approximation \\ Factor ($\alpha$)\end{tabular} & Our $\alpha$-Regret Bound                                                                                                & \begin{tabular}[c]{@{}c@{}}The Best\\Prior Bound\end{tabular}                                                       \\ \hline
\begin{tabular}[c]{@{}c@{}}Monotone SM with\\ Cardinality Constraints\end{tabular} & $1-1/e$                                                         & $\tilde{\mathcal{O}}\left(k n^\frac{1}{3} T^\frac{2}{3}\right)$                                      & $\tilde{\mathcal{O}}\left(k^\frac{4}{3} n^\frac{1}{3} T^\frac{2}{3}\right) ^*$ \\ \hline
\begin{tabular}[c]{@{}c@{}}Monotone SM with\\ Knapsack Constraints\end{tabular}    & $1/2$                                                           & $\tilde{\mathcal{O}}\left(\beta^\frac{2}{3}\Tilde{K}^\frac{1}{3} n^\frac{1}{3} T^\frac{2}{3}\right) ^\dagger$ & -                                                                           \\ \hline
\begin{tabular}[c]{@{}c@{}}Non-monotone SM\\ without Constraint\end{tabular}       & $1/2$                                                           & $\tilde{\mathcal{O}}\left(nT^\frac{2}{3}\right)$                                                     & $\tilde{\mathcal{O}}\left(nT^\frac{2}{3}\right) ^{**}$                                    \\ \hline
\end{tabular}
\caption{Table of selected results and related works. ``SM'' means submodular maximization.  $\tilde{\mathcal{O}}$ hides the logarithmic terms. Only results considering the stochastic setting with full-bandit feedback are presented. Parameters are horizon $T$, $n$ base arms, cardinality $k$, knapsack budget $B$, knapsack minimum element cost $\cmin$, knapsack ratio $\beta=\frac{B}{\cmin}$, bound on largest feasible cardinality set for knapsack case $\Tilde{K}=\min\{\beta, n\}$, and $^\ddagger$ $h$ is the discretization size. %
$^*$ is from \citep{nie2022explore}. $^{**}$ is from \citep{Fourati2023Randomized}.
}
\label{tab:related:work}
\end{table}

Table~\ref{tab:related:work} summarizes the comparison of our results with existing methods. We now briefly discuss  the  closely related works.  %

\paragraph{Adversarial  CMAB} 
The closest related works are for the adversarial  CMAB setting.  \citet{niazadeh2021online} propose a framework for transforming iterative greedy $\alpha$-approximation algorithms for offline
problems to online methods in an adversarial bandit setting, for both semi-bandit (achieving $\widetilde{O}(T^{1/2})$ $\alpha-$regret)  and full-bandit feedback (achieving $\widetilde{O}(T^{2/3})$ $\alpha-$regret).  Their framework requires the offline approximation algorithm to have an iterative greedy structure (unlike ours), satisfy a robustness property (like ours), and satisfy a property referred to as Blackwell reducibility (unlike ours). In addition to these conditions, the adaptation depends on the number of subproblems (greedy iterations) which for some algorithms can be known ahead of time (such as with cardinality constraints) but for other algorithms can only be upper-bounded.     The authors check those conditions and explicitly adapt several offline approximation algorithms. In this paper, we consider an approach for converting offline approximation algorithm to online for stochastic CMAB.  In addition to requiring less assumptions about the approximation algorithm, our procedure does not adapt individual steps of the approximation algorithm into an online method.  Instead our procedure  uses the offline algorithm as a black box.

We also note that \citet{niazadeh2021online} do not consider submodular CMAB with knapsack constraints, and thus do not verify whether any approximation algorithms for the offline problem  satisfy the required properties (of sub-problem structure or robustness or Blackwell reducibility) to be transformed, and this is an example we consider for our general framework. Consequently, in our experiments for submodular CMAB with knapsack constraints in \cref{sec:exp}, we use the algorithm in \citep{streeter2008online} designed for a knapsack constraint (in expectation) as representative of methods for the adversarial setting. Other related works for adversarial submodular CMAB with knapsack constraint are described in Appendix \ref{apdx:rad}. 

\paragraph{Stochastic Submodular CMAB with Full Bandit Feedback}   Recently,  \cite{nie2022explore} propose an algorithm for stochastic MAB with monotone submodular rewards, when there is a cardinality constraint. Further, \cite{Fourati2023Randomized} proposed an algorithm for stochastic MAB with non-monotone submodular rewards. Their algorithms are a specific adaptation of an offline greedy method.    In our work, we propose a general framework that employs the  offline algorithm as a black box (and these results becomes a special case of our approach).  While there are multiple results for semi-bandit feedback (see Appendix \ref{apdx_semi}), this paper considers full bandit feedback.

\section{Problem Statement}
\label{prob_state}

We consider sequential, combinatorial decision-making problems over a finite time horizon $T$.  Let $\Omega$ denote the ground set of base elements (arms). Let $n=|\Omega|$ denote the number of  arms.  Let $D  \subseteq 2^\Omega$ denote the subset of feasible actions (subsets), for which we presume  membership can be efficiently evaluated.  We will later consider applications with cardinality and knapsack constraints, though our methods are not limited to those.   We will use the terminologies \emph{subset} and \emph{action} interchangeably throughout the paper.  

At each time step $t$, the learner selects a feasible action $A_t \in D$.  After the subset $A_t$ is selected, the learner receives reward $f_t(A_t)$. We assume the reward $f_t$ is stochastic, bounded in $[0,1]$, and i.i.d. conditioned on a given subset.  Define the expected reward function as $f(A) = \mathbb{E}[f_t(A)]$.

The goal of the learner is to maximize the cumulative reward $\sum_{t=1}^Tf_t(A_t)$. To measure the performance of the algorithm, one common metric is to compare the learner to an agent with access to a value oracle for $f$.   However, if optimizing $f$ over $D$ is NP-hard, such a comparison would not be meaningful unless the horizon is exponentially large in the problem parameters.

If there is a known approximation algorithm $\mathcal{A}$ with approximation ratio $\alpha\in(0,1]$ for optimizing $f$ over $D$,  a more natural alternative is to evaluate the performance of a CMAB algorithm against what $\mathcal{A}$ could achieve.  Thus, we consider the %
the expected cumulative $\alpha$-regret $\mathcal{R}_{\alpha,T}$, which is the difference between $\alpha$ times the cumulative reward of the optimal subset's expected value and the average received reward, (we write $\mathcal{R}_T$ when $\alpha$ is understood from context) 
\begin{align}
    \mathbb{E}[\mathcal{R}_{T}] = \alpha Tf(\mathrm{OPT}) - \mathbb{E}\left[\sum_{t=1}^T f_t(A_t)\right],\label{eq:reg:exp1e}
\end{align} 
where  OPT is the optimal solution, i.e., $\text{OPT}\in \argmax_{A \in D } f(A)$ and the expectations are with respect to both  the  rewards %
and  actions (if random). 

\begin{commentediting3}
\cjq{a table of results?}
\end{commentediting3}

\section{Robustness of Offline  Algorithms}
\label{sec:robust}

In this section, we  introduce a criterion for an offline approximation algorithm's sensitivity to (bounded) additive perturbations to function evaluations. Investigating robustness of approximation algorithms in offline settings is valuable in its own right. %
Importantly, we will show that this property alone is sufficient to guarantee that the offline algorithm can be adapted to solve analogous combinatorial multi-armed bandit (CMAB) problems with just bandit feedback and yet achieve sub-linear regret.  Furthermore, the CMAB adaptation will not rely on any special structure of the algorithm design, instead employing it as a black box.

\begin{definition}[$(\alpha, \delta)$-Robust Approximation]\label{def:robust}
    An algorithm (possibly random) $\mathcal{A}$ is an $(\alpha, \delta)$-robust  approximation algorithm
    for the combinatorial optimization problem of maximizing a %
    function $f:2^\Omega\to \mathbb{R}$ over a finite domain $D \subseteq 2^\Omega$ if its output $S^*$ %
    using a value oracle for $\hat{f}$     satisfies the relation below with the optimal solution $\mathrm{OPT}$ under $f$, provided that for any $\epsilon >0$ that $|f(S)-\hat{f}(S)|<\epsilon$ for all $S\in D$,
    \begin{align*}
        \mathbb{E}[f(S^*)]\geq \alpha f(\mathrm{OPT})-\delta \epsilon,
    \end{align*}
    where the expectation is over the randomness of the algorithm $\mathcal{A}$.
\end{definition}

Note that the perturbed $\hat{f}$ is not required to be in the same class as $f$ (linear, quadratic, submodular, etc.). %
Thus, this definition is a stronger notion of robustness than one limited to $\hat{f}$ in the same class that have bounded $L_\infty$ distance from $f$.  

For (unstructured) multi-armed bandit problems, one can view the analogous offline algorithm with access to a value oracle for the elements as first evaluating each of the $n$ arms ($D=\{\{1\},\{2\},\dots,\{n\}\}$), using $n$ value queries total, and then evaluating  $\argmax$ over the $k$ values.  That (trivial) algorithm  is a $(1,2)$-robust approximation algorithm. (Wlog, suppose arm 1 is best in expectation, but arm 2 is chosen.  Then $f(2)\geq \hat{f}(2)-\epsilon \geq \hat{f}(1) - \epsilon \geq f(1) - 2\epsilon.$)

\begin{remark} In \cite{niazadeh2021online}, there is a related definition of robustness for offline approximation algorithms.  That definition and the subsequent offline-to-online adaptation procedure is restricted to approximation algorithms with an iterative greedy structure.  The criterion \cref{def:robust} we consider does not require the approximation algorithm to have an iterative greedy structure. 
\end{remark}

\begin{remark}
    When the  range (codomain) of $f$ and  $\hat{f}$ is within $[0,1]$, \cref{def:robust} can also handle multiplicative error $|f(S)-\hat{f}(S)|<\epsilon' f(S)$ by observing that multiplicative error implies additive error as $\epsilon' \hat{f}(S) \leq \epsilon'$. 
\end{remark}
To illustrate the utility of our proposed framework, in \cref{sec:appl-CETC:submod} we will show that several approximation algorithms from submodular maximization literature are $(\alpha, \delta)$-robust, leading to new sublinear $\alpha$-regret algorithms for related stochastic CMAB problems for those settings.

\section{C-ETC Algorithm: Offline to  Stochastic}\label{sec:alg}

In this section, we present our proposed  stochastic CMAB algorithm, \textit{Combinatorial Explore-Then-Commit} (\textsc{C-ETC}).  %
The pseudo-code is shown in \cref{alg:cetc}.   The algorithm takes an offline $(\alpha,\delta)$ robust algorithm $\mathcal{A}$ with an upper bound $N$ on the number of oracle queries by $\mathcal{A}$.  In the exploration phase, when the offline algorithm queries the value oracle for action $A$, C-ETC will play action $A$ for $m$ times, where $m$ is a constant chosen to minimizing regret. C-ETC then computes the empirical mean $\bar{f}$ of rewards for $A$ and feeds $\bar{f}$ 
 back to the offline algorithm $\mathcal{A}$. In the exploitation phase, C-ETC repeatedly  plays the  solution $S$ output from algorithm $\mathcal{A}$.   Thus, the CMAB procedure does not need $\mathcal{A}$ to have any special structure.  No careful construction is needed for the CMAB procedure beyond running $\mathcal{A}$. All that is needed is checking robustness (\cref{def:robust}) and having an upper-bound on the number of queries to the value oracle. Also, there is no over-head in terms of storage and per-round time complexities--- C-ETC is as efficient as the offline algorithm $\mathcal{A}$ itself.

\begin{algorithm}[t]
\caption{Combinatorial Explore-then-Commit}
\label{alg:cetc}
\begin{algorithmic}
    \STATE {\bfseries Input:}  horizon $T$, set $\Omega$ of $n$ base arms, an offline $(\alpha,\delta)$-robust algorithm $\mathcal{A}$, and an upper-bound $N$ on the number of  queries $\mathcal{A}$ will make to the value oracle %
    \STATE %
    \STATE  Initialize $m\gets \ceil*{\frac{\delta^{2/3}T^{2/3}\log(T)^{1/3}}{2N^{2/3}}}$
    \STATE %
    \STATE  // Exploration Phase //
    \WHILE{$\mathcal{A}$ queries the value of some  $A\subseteq \Omega$}%
        \STATE For $m$ times, play action $A$
        \STATE Calculate the empirical mean $\bar{f}$   %
        \STATE Return $\bar{f}$ to $\mathcal{A}$
    \ENDWHILE
    \STATE %
    \STATE // Exploitation Phase //
    \FOR{\emph{remaining time}}
        \STATE Play action $S$ output by algorithm $\mathcal{A}$.
    \ENDFOR
\end{algorithmic}
\end{algorithm}

Now we analyze the $\alpha$-regret for C-ETC (\cref{alg:cetc}). %
\begin{theorem} \label{thm:main}
For the sequential decision making problem defined in Section 2, with $T\geq \max\left\{N, \frac{2\sqrt{2}N}{\delta}\right\}$, the expected cumulative $\alpha$-regret of C-ETC using an $(\alpha,\delta)$-robust approximation algorithm as subroutine is at most $\mathcal{O}\left(\delta^\frac{2}{3}N^\frac{1}{3} T^\frac{2}{3}\log(T)^\frac{1}{3}\right)$, where $N$ upper-bounds the number of value oracle queries made by the offline algorithm $\mathcal{A}$.
\end{theorem} 

The detailed proof is in the supplementary material. We highlight some key steps. 

We show that with high probability, the empirical means of all actions taken during exploration phase will be within $\mathrm{rad}=\sqrt{\frac{\log T}{2m}}$ of their corresponding statistical means.  As is common in proofs for ETC methods, we refer to this occurrence as the \textit{clean event} $\mathcal{E}$. Then, using an $(\alpha,\delta)$-robust approximation algorithm as subroutine will guarantee the quality of of the set $S$ used in the exploitation phase of \cref{alg:cetc}:
\begin{align}
    \mathbb{E}[f(S)]\geq \alpha f(\mathrm{OPT})-\delta \cdot \mathrm{rad}. \nonumber
\end{align}
We then break up the expected cumulative $\alpha$-regret conditioned on the clean event $\mathcal{E}$,%
\begin{align}
     \mathbb{E}[\mathcal{R}(T)|\mathcal{E}]  &=\underbrace{\sum_{i=1}^{N} m \left(\alpha f(S^*)-\mathbb{E}[f(S_i)]\right)}_{\text{exploration phase}}  \nonumber\\
    &\qquad +\underbrace{\sum_{t=T_{N}+1}^T \left(\alpha f(S^*)-\mathbb{E}[f(S)]\right)}_{\text{exploitation phase}}, \label{eq:regr:clean:twopart}%
\end{align}
where $S_i$ is the $i$-th set algorithm $\mathcal{A}$ queries.
Using the fact that the reward is bounded between $[0,1]$, we have  
\begin{align}
    \mathbb{E}[\mathcal{R}(T)|\mathcal{E}] %
      \leq Nm+ T\delta \mathrm{rad}. \nonumber
\end{align} 
Optimizing over $m$ then results in
\begin{align}
    \mathbb{E}[\mathcal{R}(T)|\mathcal{E}] %
    &= \mathcal{O}\left(\delta^\frac{2}{3}N^\frac{1}{3} T^\frac{2}{3}\log(T)^\frac{1}{3}\right).\nonumber
\end{align}
We then show that because the clean event $\mathcal{E}$ happens with high probability, the expected cumulative regret $\mathbb{E}[\mathcal{R}(T)]$
is dominated by $\mathbb{E}[\mathcal{R}(T)|\mathcal{E}]$, which concludes the proof.

\begin{commentediting3}
\begin{remark}
\cjq{we need discussion here about alg design (also summarize in intro)  why don't we use UCB or TS?} 
\ngy{another advantage of ETC over UCB is in \url{https://arxiv.org/pdf/2002.09174.pdf}}
\end{remark}
\end{commentediting3}

\paragraph{Lower bounds:}
For the general setting we explore in this paper, with stochastic (or even adversarial) combinatorial MAB and only bandit feedback, it is unknown whether $\tilde{\mathcal{O}}(T^{1/2})$ expected cumulative $\alpha$-regret is possible (ignoring problem parameters like $n$).  For special cases, such as linear reward functions, $\tilde{\mathcal{O}}(T^{1/2})$ is known to be achievable even with bandit feedback. %
\begin{commentediting3}
    \cjq{CITE}
\end{commentediting3}
Even for the special case of submodular reward functions and a cardinality constraint, it remains an open question. 
\cite{niazadeh2021online} obtain $\tilde{\Omega}(T^{2/3})$ lower bounds for the harder setting where feedback is only available during ``exploration'' rounds chosen by the agent, who incurs an associated penalty.

\begin{remark}\label{rem:unknown_horizon} C-ETC uses knowledge of the horizon $T$ to optimize the number $m$ of samples per action.  When the time horizon $T$ is not known, we can use geometric doubling trick to extend our result to an anytime algorithm. %
We refer to the general detailed procedure in \citep{Besson2018WhatDT}. From Theorem 4 in \citep{Besson2018WhatDT}, we can show that the regret bound conserves the original $T^{2/3}\log(T)^{1/3}$ dependence with only changes in constant factors.
\end{remark}

\section{Application on Submodular Maximization} \label{sec:appl-CETC:submod}

In this section, we apply our general framework to stochastic CMAB problems with submodular rewards where only bandit feedback is available.  This application results in the first sublinear $\alpha$-regret CMAB algorithms for knapsack constraints under bandit feedback.    We begin with a brief background, and analyze the robustness of offline approximation algorithms, and then obtain problem independent regret bounds.

\subsection{Background and Definitions}
Denote the \textit{marginal gain} as $f(e|A)= f(A\cup e)-f(A)$ for any subset $A\subseteq \Omega$ and element $e\in \Omega\setminus A$.   A set function $f:2^\Omega \rightarrow \mathbb{R}$ defined on a finite ground set $\Omega$ is said to be \textit{submodular} if it satisfies the diminishing return property: for all $A\subseteq B\subseteq \Omega$, and $e\in \Omega\setminus B$, it holds that $f(e|A) \geq f(e|B)$. A set function is said to be monotonically non-decreasing if $f(A)\leq f(B)$ for all $A\subseteq B \subseteq \Omega$. Our aim is to find a set $S$ such that $f(S)$ is maximized subject to some constraints. 

For knapsack constraints, we assume that the cost function $c: \Omega \rightarrow R_{>0}$ is known and linear, so the cost of a subset is the sum of the costs of individual items: $c(A) = \sum_{v\in A}c(v)$. We denote the \textit{marginal density} as $\rho(e|A) = \frac{f(A\cup e)-f(A)}{c(e)}$ for any subset $A\subseteq \Omega$ and element $e\in \Omega\setminus A$. To simplify the presentation, we avoid the cases of trivially large budgets $B>\sum_{v\in\Omega} c(v)$ and assume all items have non-trivial costs $0<c(v)\leq B$.   A cardinality constraint is a special case with unit costs.

In the following, we consider both types of those constraints: cardinality and knapsack.  Maximizing a monotone submodular set function under a $k$-cardinality constraint is NP-hard even with a value oracle \citep{nemhauser1978analysis}. The best achievable approximation ratio with a polynomial time algorithm is $1-1/e$ \citep{nemhauser1978analysis} using $\mathcal{O}(nk)$ oracle calls. In \citep{badanidiyuru2014fast}, $1-1/e-\epsilon'$ is achieved within $\mathcal{O}(\frac{n}{\epsilon'}\log \frac{n}{\epsilon'})$ time, where $\epsilon'$ is a user selected parameter to balance accuracy and time complexity. 

Maximizing a monotone submodular set function under a knapsack constraint is consequently also NP-hard  \citep{khuller1999budgeted}. The best achievable approximation ratio with a polynomial time algorithm is $1-1/e$ \citep{Sviridenko2004ANO, khuller1999budgeted}, but that requires $\mathcal{O}(n^5)$ function evaluations, making it prohibitive for many applications.  There are other offline algorithms that achieve worse approximation ratios but are much more efficient. We adapt a $\frac{1}{2}$ approximation algorithm  \citep{yaroslavtsev2020bring} and a $\frac{1}{2}(1-1/e)$ approximation algorithm  \citep{khuller1999budgeted}, both of which use $\mathcal{O}(n^2)$ function evaluations. There is another algorithm proposed recently in  \citep{li2022submodular}, but since it queries the values of infeasible sets (i.e. sets with cost above the budget), we do not consider it.

For non-monotone submodular maximization, \cite{Buchbinder2012ATL} provides a deterministic algorithm that achieves $\frac{1}{3}$ approximation and a randomized algorithm that achieves $\frac{1}{2}$ approximation for unconstrained case, both using $\mathcal{O}(n)$ function evaluations. To show that our framework can deal with randomized offline algorithms, we adapt the $\frac{1}{2}$ approxiamtion algorithm to online full-bandit setting.

\begin{commentediting3}
\begin{remark}
\cjq{drive point home that for full bandit, even if alg A has iterative subproblem structure, if the number of subproblems $N$ is only a bound, not exact, not obvious how to handle using niazadeh's procedure (but no issue for ours) -- maybe point out for knapsack constraints if say with knapsack constraint no explicit construction so unclear if would be like streeter-golovin which cannot handle constraint (only relaxed constraint in expectation).  This is one big advantage of just accessing the offline algorithm as a black box.}
\end{remark}
\end{commentediting3}

\subsection{Offline Approximation Algorithms -- Robustness} \label{sec:submodular:offline}

For an overview of the aforementioned offline approximation algorithms for submodular optimization, please refer \cref{apdx_offlineover}.  We next state our results on $(\alpha,\delta)$-robustness of the offline algorithms considered. The assumption of complete/noiseless access to a value oracle is often a strong assumption for real world applications.  Thus, even for offline applications, it is worthwhile knowing how robust an algorithm is.  So the following results are relevant even in the offline setting.    For the CMAB setting we consider, robustness is also a sufficient property to guarantee a no-regret  adaptation of the offline algorithm.   Detailed proofs are included in \cref{sec:offline:robust}. %

\begin{theorem}[Corollary 4.3 of \citep{nie2022explore}] \label{thm:cardinality:greedy:robust}
    \textsc{Greedy} in \citep{nemhauser1978analysis} is a $(1-\frac{1}{e},2k)$-robust approximation algorithm for monotone submodular maximization under a $k$-cardinality constraint.
\end{theorem}

\begin{commentediting3}
   \cjq{explicitly say that those authors directly adapted the method; implicit in their regret analysis is the $\delta=2k$} 
\end{commentediting3}

\begin{theorem} \label{thm:threshold:greedy:robust}
    \textsc{ThresholdGreedy} \citep{badanidiyuru2014fast} is a $(1-\frac{1}{e}-\epsilon',2(2-\epsilon')k)$-robust approximation algorithm for monotone submodular maximization under a $k$-cardinality constraint.
\end{theorem}

\begin{theorem}\label{thm:partial:enum:greedy:robust}
    \textsc{PartialEnumeration} \citep{Sviridenko2004ANO,khuller1999budgeted} is a $(1-\frac{1}{e},4+2\Tilde{K}+2\beta)$-robust approximation algorithm for monotone submodular maximization under a knapsack constraint.
\end{theorem}

\begin{theorem}\label{thm:greedy:max:robust}
    \textsc{Greedy+Max} \citep{yaroslavtsev2020bring} is a $(\frac{1}{2},\frac{1}{2}+\Tilde{K}+2\beta)$-robust approximation algorithm for monotone submodular maximization problem under a knapsack constraint.
\end{theorem}

\begin{theorem}\label{thm:greedy:robust}
    \textsc{Greedy+} \citep{khuller1999budgeted} is a $(\frac{1}{2}(1-\frac{1}{e}),2+\Tilde{K}+\beta)$-robust approximation algorithm for monotone submodular maximization problem under a knapsack constraint.
\end{theorem}

\begin{theorem}[Corollary 2 of \citep{Fourati2023Randomized}]\label{thm:RandomUSM:robust}
    \textsc{RandomizedUSM} \citep{Buchbinder2012ATL} is a $(\frac{1}{2},\frac{5}{2}n)$-robust approximation algorithm for unconstrained non-monotone submodular maximization problem.
\end{theorem}

\begin{remark} \label{rem:alpha-delta:differences}  
    For the offline setting, \textsc{Greedy+Max} is superior to \textsc{Greedy+}, as it achieves a better $\alpha$ approximation ratio with the same calls to the value oracle.  However, their $(\alpha,\delta)$ pairs are incomparable, as for $\beta>1.5$ (with  $\beta=1$ corresponding to a cardinality constraint), \textsc{Greedy+} has a smaller $\delta$ (thus more robust) which affects exploration time in their adaptations and in turn affects their regret.
\end{remark}

To illustrate the robustness analysis, we highlight some key steps for the proof of \cref{thm:greedy:max:robust} for \textsc{Greedy+Max}. Let $o_1\in \argmax_{e:e\in \text{OPT}}c(e)$ denote the most expensive element in OPT. During the $i$th iteration of the greedy process, having previously selected the set $G_{i-1}$ with $i-1$ elements, it will select the element $g_i$ with highest marginal density (based on surrogate function $\hat{f}$) among  feasible elements, 
\begin{align}
    g_i = \argmax_{e :\ e \in \Omega \backslash G_{i-1}} \hat{\rho}(e|G_{i-1}). %
\end{align} Inspired by the proof techniques in \citep{yaroslavtsev2020bring}, we consider the last item added by the greedy solution (based on surrogate function $\hat{f}$) before the cost of this solution exceeds $B-c(o_1)$. Let $G_i$ denote the set selected by \textsc{Greedy} that has cardinality $i$ and denote the constituent elements as $G_i=\{g_1,\cdots, g_i\}$. Denote $G_\ell$ as the largest greedy sequence that consumes less than $B-c(o_1)$ of the budget $B$, so $c(G_\ell) \leq B-c(o_1) < c(G_{\ell+1})$. Let $a_i$ denote the element selected to augment with the greedy solution $G_i$, i.e.,  $a_i=\argmax_{e\in \Omega\setminus G_i}\hat{f}(e|G_i)$, and $S_i$ denote the augmented set at $i$-th iteration. Denote the marginal gain as $\hat{f}(e|S):=\hat{f}(S\cup e)-\hat{f}(S)$ and the marginal density as $\hat{\rho}(e|S):=\frac{\hat{f}(S\cup e)-\hat{f}(S)}{c(e)}$. 

Before proving the theorem, we first show the following lemma, adapted from \citep{yaroslavtsev2020bring} for the setting where $f$ can be evaluated exactly. The lemma upper-bounds $f(\mathrm{OPT})$ as the sum of two quantities, so at least one must have value of $\frac{1}{2}f(\mathrm{OPT})$.  Relating each quantity to the output of the algorithm leads to the approximation ratio $\alpha=\frac{1}{2}$. We bound how the error due to working with $\hat{f}$ instead of $f$ compounds in addition to new complications arising such as the marginal density of greedily selected elements increasing (i.e. we may have $\hat{\rho}(g_{i+1}|G_i)>\hat{\rho}(g_{i}|G_{i-1})$ and/or $\rho(g_{i+1}|G_i)>\rho(g_{i}|G_{i-1})$) where when when $\epsilon=0$ it is guaranteed to be non-increasing.

\begin{lemma} [\textsc{Greedy+Max} inequality] \label{lem:gm:ineq}
     For $i \in \{0,1,\cdots, \ell\}$, the following inequality holds:
    \begin{align}
        \hat{f}(G_i\cup o_1) + %
        \max\{0,\hat{\rho}(g_{i+1}|G_i)\}(B-c(o_1)) %
        &\geq f(\mathrm{OPT})-(2\Tilde{K}-1)\epsilon. \nonumber
    \end{align}
\end{lemma}

\begin{proof}[\cref{lem:gm:ineq}]
Recall that from the definition of $\hat{f}$, we have $|\hat{f}(S)-f(S)|\leq \epsilon$ for any evaluated set $S$ and some $\epsilon > 0$. Consequently, we have for any  $i\in \{0,1,\cdots, \ell\}$, 
\begin{align}
    |\hat{f}(G_i)-f(G_i)| \leq \epsilon. \label{eq:gdiff}
\end{align}
Now we evaluate the set $G_i\cup o_1$. 

\begin{itemize}
    \item \textbf{Case 1:} If $o_1$ has already been added, $o_1\in G_i$, then
        \begin{align}
            |\hat{f}(G_i\cup o_1)-f(G_i\cup o_1)| = |\hat{f}(G_i)-f(G_i)| \leq \epsilon. \nonumber
        \end{align}

    \item \textbf{Case 2}: If $o_1 \notin G_i$, then $\hat{f}(G_i\cup o_1)$ is evaluated in iteration $i+1$. This iteration $i+1$ does exist\footnote{\label{foot:supp:evalclean} For $(\alpha,\delta)$ robustness alone, this point is not necessary due to the assumption of $|f(S) - \hat{f}(S)|\leq \epsilon$ for all $S\subseteq \Omega$.  For the regret bound proof of our proposed C-ETC method in \cref{sec:apdx:prf:main-thm}, the ``clean event'' (corresponding to concentration of empirical mean of set values around their statistical means) will only imply concentration for those actions taken and thus for which empirical estimates exist.} because for any $i\in \{0,1,\cdots, \ell\}$, we only used less than $B-c(o_1)$ budget. For the remaining budget, at least $o_1$ can still fit into the budget so $G_i\cup o_1$ will be evaluated in iteration $i+1$. In this case, we still have
        \begin{align}
            |\hat{f}(G_i\cup o_1)-f(G_i\cup o_1)| \leq \epsilon. \nonumber
        \end{align}
\end{itemize}

Combining these two cases, we have
\begin{align}
    |\hat{f}(G_i\cup o_1)-f(G_i\cup o_1)| \leq \epsilon. \label{eq:Go1diff}
\end{align}
Also, for any evaluated action in iteration $i+1$, namely the actions $\{G_i \cup e|e\in \Omega\setminus G_i \text{ and } c(e)+c(G_i)\leq B\}$, we have
\begin{align}
    \rho(e|G_i) &= \frac{f(G_i \cup e)-f(G_i)}{c(e)} \nonumber\\
    &\leq \frac{\hat{f}(G_i \cup e)-\hat{f}(G_i)}{c(e)} + \frac{2\epsilon}{c(e)} \nonumber\\
    &= \hat{\rho}(e|G_i)+\frac{2\epsilon}{c(e)}. \label{eq:rho:concentration}
\end{align}
Then we have
\begin{align}
    f(\mathrm{OPT}) &\leq f(G_i\cup \mathrm{OPT})  \tag{Monotonicity of $f$}\\
    &\leq f(G_i\cup o_1) + f(\mathrm{OPT}\setminus(G_i\cup o_1)|G_i\cup o_1) \nonumber\\
    & \leq f(G_i\cup o_1) + \sum_{e\in \mathrm{OPT}\setminus (G_i\cup o_1)} f(e|G_i\cup o_1) \tag{Submodularity of $f$}\\
    & \leq \hat{f}(G_i\cup o_1)+ \epsilon +\sum_{e\in \mathrm{OPT}\setminus (G_i\cup o_1)} c(e)\rho(e|G_i\cup o_1) . \label{eq:prf:rbst:fopt:4}
\end{align} where \eqref{eq:prf:rbst:fopt:4} uses \eqref{eq:Go1diff}.
Since we picked iteration $i$ such that $c(G_i)\leq B-c(o_1)$, then all items in $\mathrm{OPT}\setminus (G_i\cup o_1)$ still fit, as $o_1$ is the largest item in OPT. Since the greedy algorithm always selects the item with the largest marginal density with respect to the surrogate function $\hat{f}$, $g_i = \argmax_{e\in \Omega\setminus G_i} \hat{\rho}(e|G_i)$, thus we have  %
\begin{align}
    \hat{\rho}(g_{i+1}|G_i) =\max_{e\in \Omega\setminus G_i} \hat{\rho}(e|G_i) \geq \max_{e\in \Omega\setminus (G_i\cup o_1)} \hat{\rho}(e|G_i). \label{eq:hprime}
\end{align}
Hence, continuing with \eqref{eq:prf:rbst:fopt:4}, 
\begin{align}
    f(\mathrm{OPT}) &\leq \hat{f}(G_i\cup o_1)+ \epsilon  +\left(\sum_{e\in \mathrm{OPT}\setminus (G_i\cup o_1)} c(e)\rho(e|G_i\cup o_1) \right) \nonumber\\
    &\leq \hat{f}(G_i\cup o_1)+\epsilon +\sum_{e\in \mathrm{OPT}\setminus (G_i\cup o_1)} c(e)\rho(e|G_i)  \tag{Submodularity}\\
    &\leq \hat{f}(G_i\cup o_1)+\epsilon +\sum_{e\in \mathrm{OPT}\setminus (G_i\cup o_1)} c(e)\left(\hat{\rho}(e|G_i)+\frac{2\epsilon}{c(e)}\right)  \tag{using \eqref{eq:rho:concentration}}\\
    &\leq \hat{f}(G_i\cup o_1)+\epsilon +\sum_{e\in \mathrm{OPT}\setminus (G_i\cup o_1)}\bigg( c(e)\hat{\rho}(e|G_i)\bigg)+2\epsilon |\mathrm{OPT}\setminus (G_i\cup o_1)|  \nonumber\\
    &\leq \hat{f}(G_i\cup o_1) +\epsilon+ \hat{\rho}(g_{i+1}|G_i)\sum_{e\in \mathrm{OPT}\setminus (G_i\cup o_1)}\bigg( c(e)\bigg)+2\epsilon |\mathrm{OPT}\setminus (G_i\cup o_1)|  \tag{Using \eqref{eq:hprime}}\\
    &\leq \hat{f}(G_i\cup o_1) +\epsilon+ \hat{\rho}(g_{i+1}|G_i)c(\mathrm{OPT}\setminus (G_i\cup o_1))+2\epsilon |\mathrm{OPT}\setminus (G_i\cup o_1)|  \nonumber\\
    &\leq \hat{f}(G_i\cup o_1)+\epsilon + \max\{0,\hat{\rho}(g_{i+1}|G_i)\}c(\mathrm{OPT}\setminus (G_i\cup o_1))+2\epsilon |\mathrm{OPT}\setminus (G_i\cup o_1)|  \nonumber\\
    &\leq \hat{f}(G_i\cup o_1)+\epsilon + \max\{0,\hat{\rho}(g_{i+1}|G_i)\}(B-c(o_1))+2\epsilon |\mathrm{OPT}\setminus (G_i\cup o_1)|  \nonumber\\
    &\leq \hat{f}(G_i\cup o_1) + \max\{0,\hat{\rho}(g_{i+1}|G_i)\}(B-c(o_1))+(2 \Tilde{K}-1)\epsilon . \nonumber
\end{align} 
Rearranging terms gives the desired result. 
\end{proof}
Now we are ready to prove \cref{thm:greedy:max:robust}. 
\begin{proof}[\cref{thm:greedy:max:robust}] 
Applying \cref{lem:gm:ineq} (\textsc{Greedy+Max} inequality) for $i=\ell$, and recalling that $\ell$ is chosen as the index of the last greedy set such that $c(G_\ell) \leq B-c(o_1) < c(G_{\ell+1})$,
\begin{align}
    \hat{f}(G_\ell \cup o_1) + \max\{0,\hat{\rho}(g_{\ell+1}|G_\ell)\}(B-c(o_1)) \geq f(\mathrm{OPT})-(2\Tilde{K}-1)\epsilon. \label{eq:case:commonsum}
\end{align}
From \eqref{eq:case:commonsum}, we will next argue at least one of the terms in the left hand side must be large.  We will consider cases for the two terms being large.  To minimize the worst-case additive error term from the  cases, we will split the cases into whether $\hat{f}(G_\ell \cup o_1)$ is larger than or equal to $\frac{1}{2}f(\mathrm{OPT})-(\Tilde{K}-\frac{1}{2}+\gamma)\epsilon$, or $\max\{0,\hat{\rho}(g_{\ell+1}|G_\ell\}(B-c(o_1))$ is larger than or equal to $\frac{1}{2}f(\mathrm{OPT})-(\Tilde{K}-\frac{1}{2}-\gamma)\epsilon$, where $\gamma$ will be selected later to minimize the additive error $\delta$ coefficient.

\textbf{Case 1:} If $\hat{f}(G_\ell \cup o_1)\geq\frac{1}{2}f(\mathrm{OPT})-(\Tilde{K}-\frac{1}{2}+\gamma)\epsilon$, recall that $a_\ell$ as the element selected to augment with the greedy solution $G_\ell$, $a_\ell=\argmax_{e\in \Omega\setminus G_\ell}\hat{f}(e|G_\ell)$, then
\begin{align}
    \hat{f}(G_\ell \cup a_\ell) &\geq \hat{f}(G_\ell \cup o_1) \nonumber\\
    &\geq \frac{1}{2}f(\mathrm{OPT})-\left(\Tilde{K}-\frac{1}{2}+\gamma\right)\epsilon. \label{eq:case1:1}
\end{align}
The set $S$ that the algorithm selects in the end will be the set with the highest mean (based on surrogate function $\hat{f}$) among all those evaluated (both sets in the greedy process and their augmentations).  Also, its observed value $\hat{f}(S_\ell)$ is at most $\epsilon$ above $f(S)$.  Thus %
\begin{align}
    f(S) &\geq \hat{f}(S)-\epsilon \nonumber\\
    &\geq \hat{f}(G_\ell \cup a_\ell) - \epsilon \nonumber\\
    &\geq \frac{1}{2}f(\mathrm{OPT})-\left(\Tilde{K}+\frac{1}{2}+\gamma\right)\epsilon. \tag{using \eqref{eq:case1:1}}
\end{align}
\textbf{Case 2(a):} If $\max\{0,\hat{\rho}(g_{\ell+1}|G_\ell)\}(B-c(o_1)) \geq \frac{1}{2}f(\mathrm{OPT})-(\Tilde{K}-\frac{1}{2}-\gamma)\epsilon$ and $\hat{\rho}(g_{\ell+1}|G_\ell)>0$, rearranging we have
\begin{align}
    \hat{\rho}(g_{\ell+1}|G_\ell) \geq \frac{f(\mathrm{OPT})}{2(B-c(o_1))}-\frac{(\Tilde{K}-\frac{1}{2}-\gamma)\epsilon}{B-c(o_1)}. \label{eq:hprime:error}
\end{align}
Then, 
\begin{align}
    \hat{f}(G_\ell) &= \hat{f}(G_\ell) - \hat{f}(G_{\ell-1}) + \hat{f}(G_{\ell-1}) +\cdots-\hat{f}(G_1)+\hat{f}(G_1)-\hat{f}(G_0)\tag{telescoping sum; $G_0=\emptyset$, $\hat{f}(G_0):=0$}\\
    &= \sum_{j=1}^{\ell-1} \hat{f}(g_{j+1}|G_j) \tag{Definition of $\hat{f}(\cdot | \cdot)$}\\
    &= \sum_{j=0}^{\ell-1}\hat{\rho}(g_{j+1}|G_j)c(g_{j+1}) \tag{Definition of $\hat{\rho}(\cdot | \cdot)$}\\
    &\geq \sum_{j=0}^{\ell-1}\hat{\rho}(g_{\ell+1}|G_j)c(g_{j+1}) \tag{greedy choice of $g_{j+1}$}\\
    &\geq \sum_{j=0}^{\ell-1}\bigg(\rho(g_{\ell+1}|G_j) - \frac{2\epsilon}{c(g_{\ell+1})} \bigg)c(g_{j+1}) \nonumber\\
    &\geq \sum_{j=0}^{\ell-1}\bigg(\rho(g_{\ell+1}|G_\ell) - \frac{2\epsilon}{c(g_{\ell+1})} \bigg)c(g_{j+1}) \tag{submodularity of $f$}\\
    &= \bigg(\rho(g_{\ell+1}|G_\ell) - \frac{2\epsilon}{c(g_{\ell+1})} \bigg)c(G_\ell) \tag{simplifying}\\
    &\geq \bigg(\hat{\rho}(g_{\ell+1}|G_\ell) - \frac{4\epsilon}{c(g_{\ell+1})} \bigg)c(G_\ell) \nonumber\\
    &\geq \hat{\rho}(g_{\ell+1}|G_\ell)c(G_\ell)- 4\beta\epsilon.\label{eq:case2:h:mid}
\end{align}
Recalling that $\ell$ is chosen as the index of the last greedy set that has a remaining budget as big as the cost of the heaviest element in $\mathrm{OPT}$,  $c(G_\ell) \leq B-c(o_1) < c(G_{\ell+1})$,
\begin{align}
    \hat{f}(G_{\ell+1}) &= \hat{f}(G_\ell \cup g_{\ell+1}) \nonumber\\
    &= \hat{f}(G_\ell) + c(g_{\ell+1})\hat{\rho}(g_{\ell+1}|G_\ell) \nonumber\\
    &\geq \bigg(\hat{\rho}(g_{\ell+1}|G_\ell)c(G_\ell)- 4\beta\epsilon\bigg)+ c(g_{\ell+1})\hat{\rho}(g_{\ell+1}|G_\ell) \tag{ from \eqref{eq:case2:h:mid}}\\
    &= \hat{\rho}(g_{\ell+1}|G_\ell) c(G_{\ell+1})   - 4\beta\epsilon \tag{simplifying}\\
    &\geq \frac{\frac{1}{2}f(\mathrm{OPT})-(\Tilde{K}-\frac{1}{2}-\gamma)\epsilon}{B-c(o_1)}  c(G_{\ell+1})   - 4\beta\epsilon \tag{case 2 condition}\\
    &\geq \frac{1}{2}f(\mathrm{OPT})-(\Tilde{K}-\frac{1}{2}-\gamma)\epsilon  - 4\beta\epsilon \tag{$\ell$ chosen so that $c(G_{\ell+1})>B-c(o_1)$}\\
    &= \frac{1}{2}f(\mathrm{OPT})-\left(\Tilde{K}-\frac{1}{2}-\gamma + 4\beta\right)\epsilon. \label{eq:case2:1}
\end{align}
For the output set $S$ selected by the algorithm, we have
\begin{align}
    f(S) &\geq \hat{f}(S)-\epsilon \nonumber\\
    &\geq \hat{f}(G_{\ell+1})-\epsilon \nonumber\\ 
    &\geq \frac{1}{2}f(\mathrm{OPT})-\left(\Tilde{K}+\frac{1}{2}-\gamma + 4\beta\right)\epsilon. \tag{using \eqref{eq:case2:1}}
\end{align}
\textbf{Case 2(b):} If $\max\{0,\hat{\rho}(g_{\ell+1}|G_\ell)\}(B-c(o_1)) \geq \frac{1}{2}f(\mathrm{OPT})-(\Tilde{K}-\frac{1}{2}-\gamma)\epsilon$ and $\hat{\rho}(g_{\ell+1}|G_\ell)\leq 0$, then the set $S$ that the algorithm selects at the end satisfies
\begin{align}
    f(S) &\geq 0 \nonumber\\
    &\geq \frac{1}{2}f(\mathrm{OPT})-(\Tilde{K}-\frac{1}{2}-\gamma)\epsilon \tag{Case 2(b) condition}\\
    &\geq \frac{1}{2}f(\mathrm{OPT})-(\Tilde{K}-\frac{1}{2}-\gamma + 4\beta)\epsilon. \nonumber
\end{align}
Thus, combining cases 1 and 2, and selecting $\gamma = 2\beta$, the additive $\frac{1}{2}$-approximation error we get by the modified Greedy+Max algorithm is at most $(\frac{1}{2}+\Tilde{K}+2\beta)\epsilon$, which concludes the proof. 
\end{proof}

\subsection{CMAB algorithms for Online Submodular Maximization}
Now that we have analyzed the robustness of several offline algorithms, we can  invoke \cref{thm:main} to bound the expected cumulative $\alpha$ regret for stochastic CMAB adaptations that rely only on bandit feedback.  We name the adapted algorithms as C-ETC-N, C-ETC-Ba for cardinality constraint, C-ETC-S, \etckgk\ , \etckga\ for knapsack constraint, and C-ETC-Bu for unconstrained, respectively, based on which offline algorithm it is adapted from (using the first author's last name); which are in order \citep{nemhauser1978analysis, badanidiyuru2014fast, Sviridenko2004ANO, khuller1999budgeted, yaroslavtsev2020bring, Buchbinder2012ATL}.  \textsc{PartialEnumeration} was first proposed and analyzed by \citet{khuller1999budgeted} for maximum coverage problems and then analyzed by \citet{Sviridenko2004ANO} for monotone submodular functions.  To distinguish CMAB adaptations of \textsc{Greedy+} and \etckgk, both proposed in \citep{khuller1999budgeted}, we use C-ETC-S for the adaption of \textsc{PartialEnumeration}. %
The following corollaries hold immediately: 
\begin{corollary} 
    For an online monotone submodular maximization problem under a cardinality constraint, the expected cumulative $(1-1/e)$-regret of C-ETC-N is at most $\mathcal{O}\left(k n^\frac{1}{3} T^\frac{2}{3}\log(T)^\frac{1}{3}\right)$ for $T\geq \max\left\{k,\sqrt{2}\right\}n$, using $N=kn$ as an upper bound of the number of function evaluations for the corresponding offline algorithm.  
\end{corollary}
\begin{remark}
This result improves upon the result from \citep{nie2022explore} by a factor of $k^{\frac{1}{3}}$ despite our use of a generic framework.
\end{remark}
\begin{corollary} \label{cor:threshold:greedy}
    For an online monotone submodular maximization problem under a cardinality constraint, the expected cumulative $(1-1/e-\epsilon')$-regret of C-ETC-Ba is at most $\mathcal{O}\left(k^\frac{2}{3}n^\frac{1}{3}(\epsilon')^\frac{1}{3}(\log \frac{n}{\epsilon'})^\frac{1}{3} T^\frac{2}{3}\log(T)^\frac{1}{3}\right)$ for $T\geq \max\left\{\frac{n}{\epsilon'}, \frac{\sqrt{2}n}{(2-\epsilon')\epsilon' k}\right\}\log \frac{n}{\epsilon'}$, using $N=\frac{n}{\epsilon'} \log \frac{n}{\epsilon'}$ as an upper bound of the number of function evaluations for the corresponding offline algorithm.
\end{corollary}
\begin{corollary}
    For an online monotone submodular maximization problem under a knapsack constraint, the expected cumulative $(1-1/e)$-regret of C-ETC-S is at most $\mathcal{O}\left(\beta^\frac{2}{3}\Tilde{K}^\frac{1}{3} n^\frac{4}{3} T^\frac{2}{3}\log(T)^\frac{1}{3}\right)$ for $T\geq \Tilde{K}n^4$, using $N=\Tilde{K}n^4$ as an upper bound of the number of function evaluations for the corresponding offline algorithm.
\end{corollary}
\begin{corollary}
    For an online monotone submodular maximization problem under a knapsack constraint, the expected cumulative $\frac{1}{2}$-regret of \etckga\ is at most $\mathcal{O}\left(\beta^\frac{2}{3}\Tilde{K}^\frac{1}{3} n^\frac{1}{3} T^\frac{2}{3}\log(T)^\frac{1}{3}\right)$ for $T\geq \Tilde{K}n$, %
    using $N=\Tilde{K}n$ as an upper bound of the number of function evaluations for the corresponding offline algorithm.
\end{corollary}

\begin{corollary}
    For an online monotone submodular maximization problem under a knapsack constraint, the expected cumulative $\frac{1}{2}(1-\frac{1}{e})$-regret of \etckgk\ is at most $\mathcal{O}\left(\beta^\frac{2}{3}\Tilde{K}^\frac{1}{3} n^\frac{1}{3} T^\frac{2}{3}\log(T)^\frac{1}{3}\right)$ for $T\geq \Tilde{K}n$, %
    using $N=\Tilde{K}n$ as an upper bound of the number of function evaluations for the corresponding offline algorithm.
\end{corollary}

\begin{corollary}
    For an online unconstrained non-monotone submodular maximization problem, the expected cumulative $\frac{1}{2}$-regret of C-ETC-Bu is at most $\mathcal{O}\left(nT^\frac{2}{3}\log(T)^\frac{1}{3}\right)$ for $T\geq 4n$, using $N=4n$ as an upper bound of the number of function evaluations for the corresponding offline algorithm. %
\end{corollary}
\begin{remark}
    This result recovers the result from \citep{Fourati2023Randomized}.
\end{remark}
\textbf{Storage and Per-Round Time Complexities:} C-ETC-Y and C-ETC-K have low storage complexity and per-round time-complexity.  During exploitation, only the indices of at most $\Tilde{K}$ base arms are needed in memory and does not need any computation. During exploration, they just need to update the empirical mean for the current action at time $t$, which can be done in $\mathcal{O}(1)$ time.  It additionally stores the highest empirical density so far in the current iteration of the greedy routine and its associated base arm (\etckgk\ needs to store one more arm and \etckga\ an additional $\mathcal{O}(\Tilde{K})$ storage is needed to store the augmented set). Thus, C-ETC-Y and C-ETC-K have $\mathcal{O}(\Tilde{K})$ storage complexity  and $\mathcal{O}(1)$ per-round time complexity. For comparison, the algorithm proposed by \cite{streeter2008online} for an averaged knapsack constraint  in the adversarial setting uses $\mathcal{O}(n\Tilde{K})$ storage complexity and $\mathcal{O}(n)$ per-round time complexity. Some comments on lower bound are given in Appendix \ref{apdx_lb_sub}.

\section{Experiments}\label{sec:exp}

In this section, we conduct experiments on real world data with a Budgeted Influence Maximization (BIM) and Song Recommendation (SR). Both of these are applications of stochastic CMAB with submodular rewards under a knapsack constraint. For experiments with monotone submodular functions with cardinality canstraints, we refer to \citep{nie2022explore}. For experiments with non-monotone submodular functions, we refer to \citep{Fourati2023Randomized}. Those algorithms are the same as appying our framework to the corresponding problems except for minor  changes in the choice of parameters.   There are three adaptions we considered in \cref{sec:appl-CETC:submod} for knapsack constraint. Since the time complexity for \textsc{PartialEnumeration} is much larger than the other two offline algorithms we consider, it will use at least $T\approx 10^8$ for C-ETC-S to finish exploration. For this reason, we do not consider C-ETC-S in the experiments. To our knowledge, our work is the first to consider these applications with only bandit feedback available. 

\paragraph{Baseline:}
The only other algorithm designed for combinatorial MAB with general submodular rewards, under a knapsack constraint, and using full-bandit feedback is \textbf{Online Greedy with opaque feedback model (OG$^\text{o}$)} proposed by \citet{streeter2008online} for the adversarial setting. 
However, OG$^\text{o}$ only satisfies the knapsack constraint in expectation, while our algorithms C-ETC-K ands C-ETC-Y satisfies a strict constraint (i.e. every action $A_t$ must be under budget).  See Appendix \ref{implimentation:ogo} for more details about OG$^\text{o}$ %
and its implementation. 

In \cref{sec:appl-CETC:submod}, we used $N=\Tilde{K}n$ as an upper bound on the number of function evaluations for both C-ETC-K and C-ETC-Y, where $n$ is the number of base arms and $\Tilde{K}$ is an upper bound of the cardinality of any feasible set. When the time horizon $T$ is small, it is possible that the exploration phase will not finish due to the formula being optimized for $m$ (the number of plays for each action queried by $\mathcal{A}$) uses a loose bound on the exploitation time.   When this is the case, we select the largest $m$ (closest to the formula) for which we can guarantee that exploration will finish. For details, see \cref{apdx:smallT}.

\begin{figure}[t]%
    \centering
    \subfloat[]{\label{fig:im:a}{\includegraphics[width=0.35\linewidth]{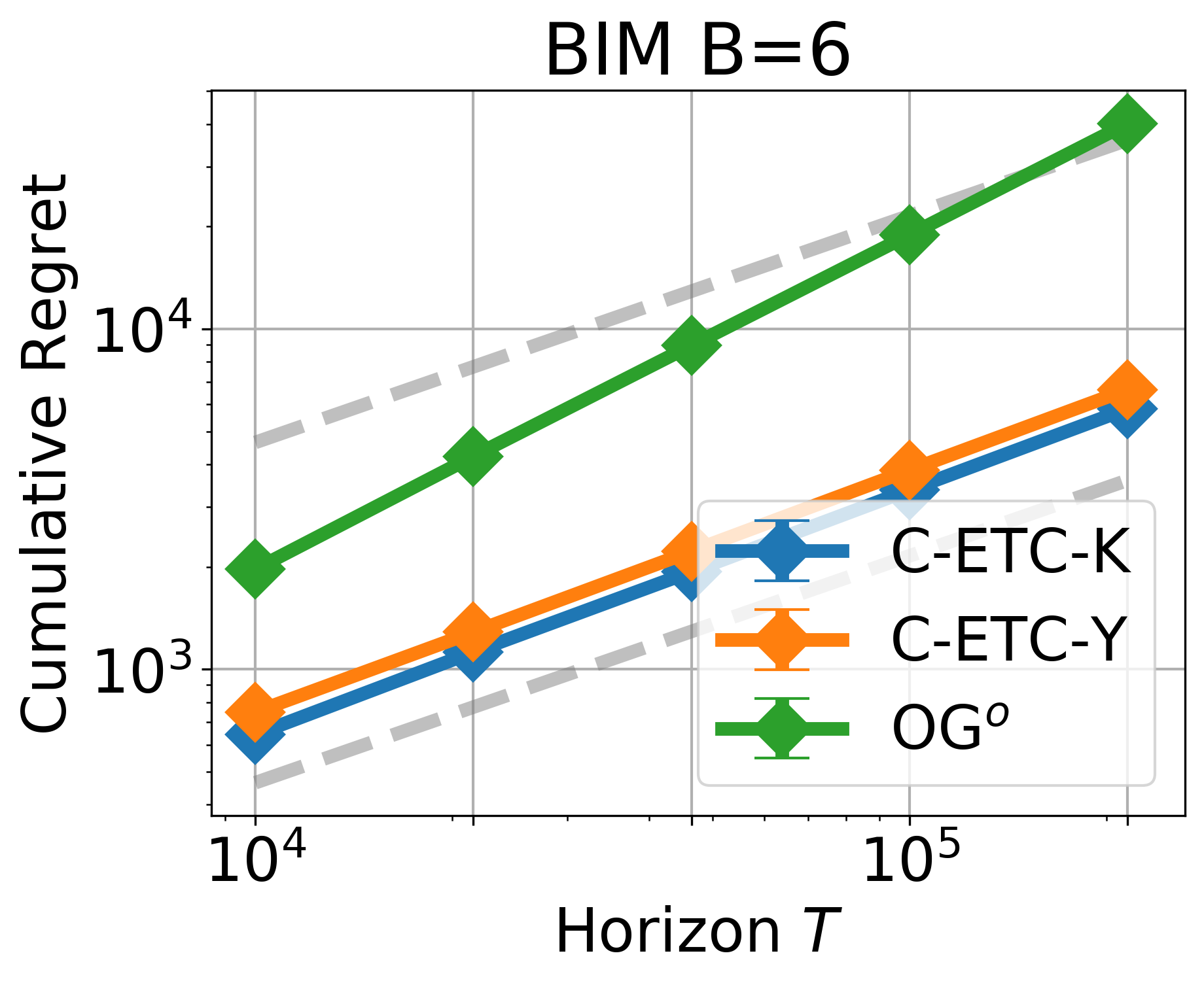} }}\qquad%
    \subfloat[]{\label{fig:im:b}{\includegraphics[width=0.35\linewidth]{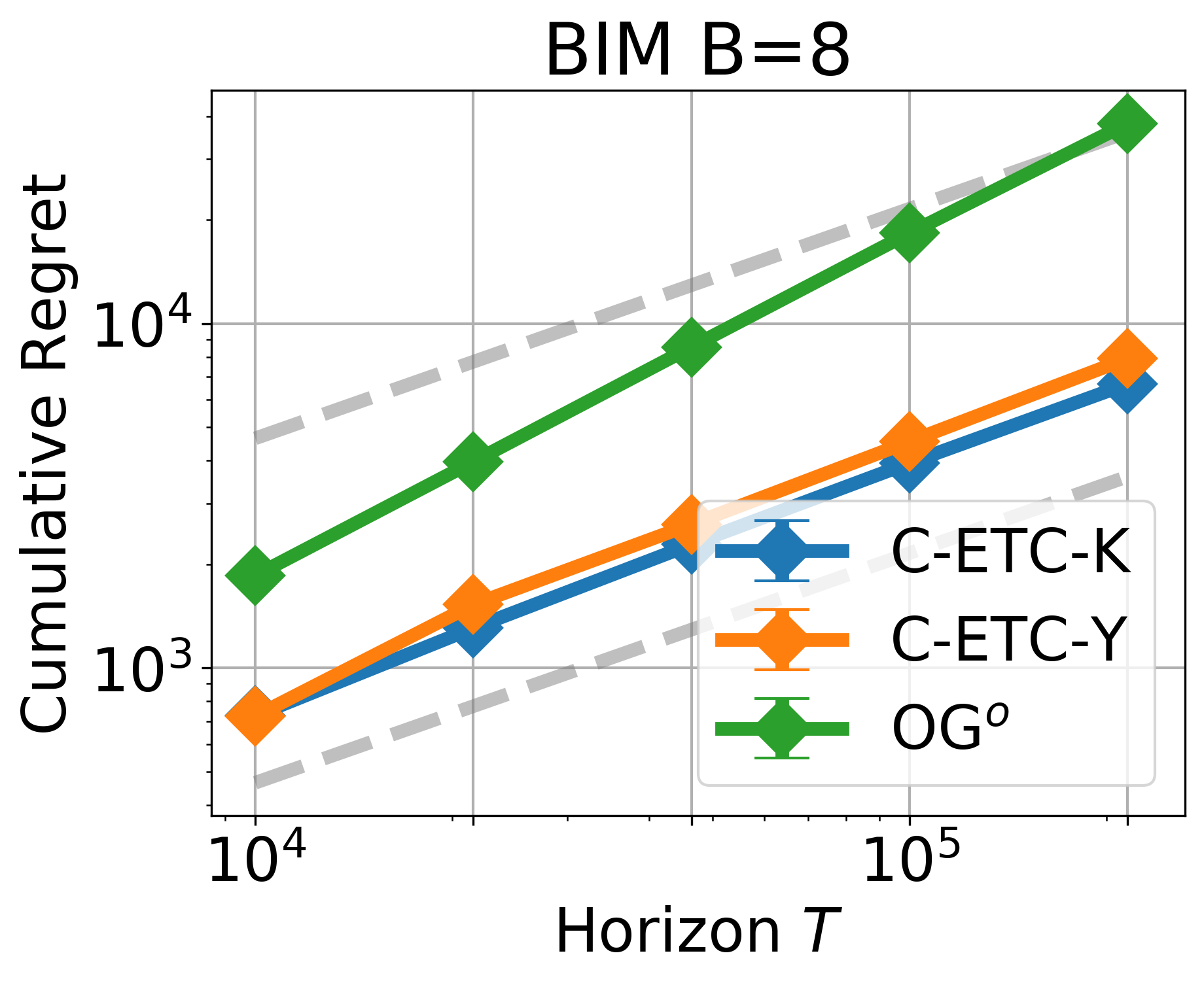} }}%
    \vspace{.5cm}\\
    \subfloat[]{\label{fig:im:c}{\includegraphics[width=0.35\linewidth]{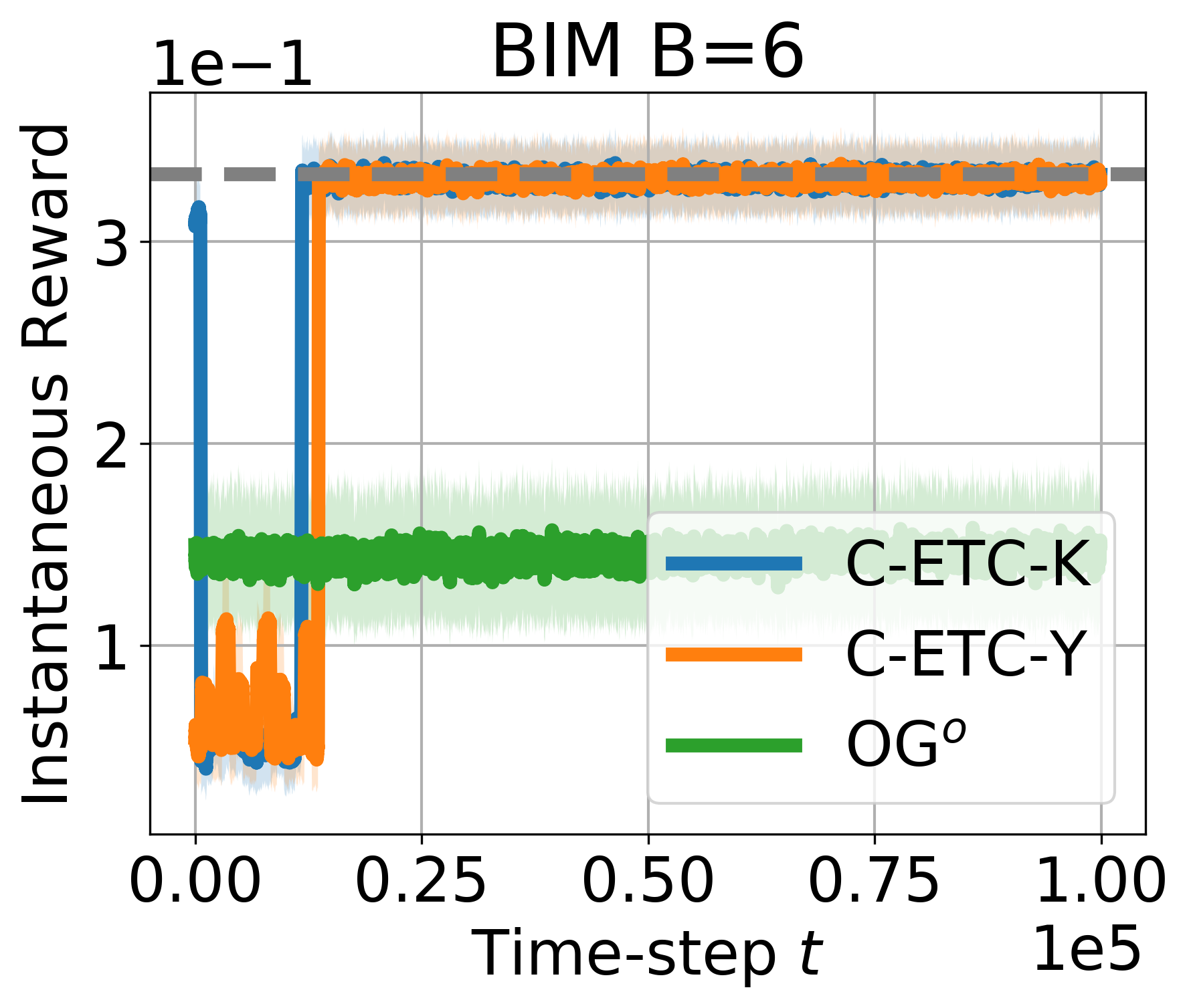} }}\qquad%
    \subfloat[]{\label{fig:im:d}{\includegraphics[width=0.35\linewidth]{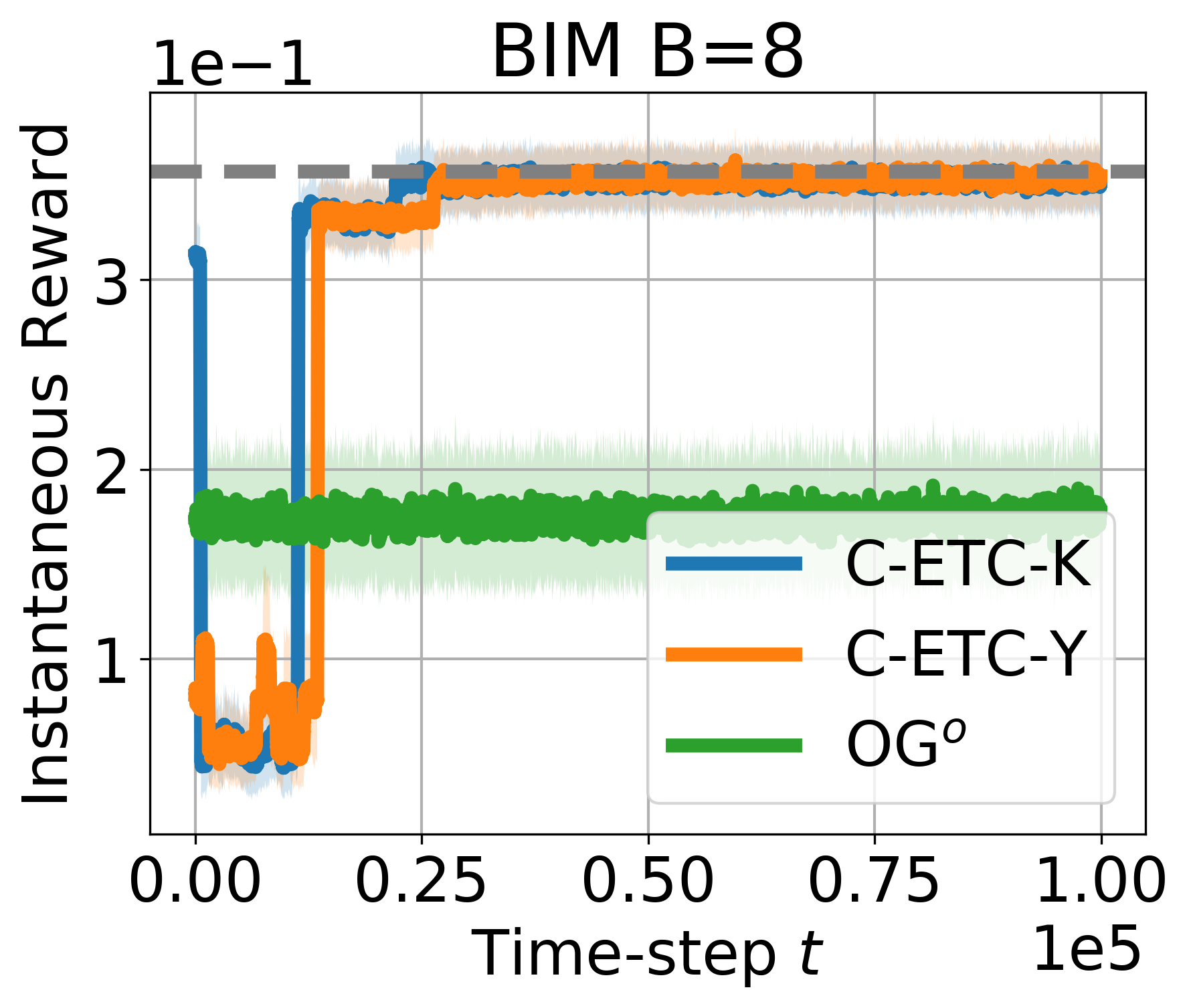} }}%
    \caption{\small Plots for budgeted influence maximization (BIM) example. (a) and (b) are results for cumulative regret as a function of the time horizon $T$ for budgets $B=6$ and $B=8$ respectively. (c) and (d) are instantaneous reward plots (smoothed with a moving average over a window size of 100) as a function of time $t$. The gray dashed lines in (a) and (b) represent $y = aT^{2/3}$ for various values of $a$ for visual reference. The gray dashed lines in  (c) and (d) represent expected rewards for the action chosen by an offline greedy algorithm.}%
    \label{fig:im}%
\end{figure}

\subsection{Experiments with Budgeted Influence Maximization}
We first conduct experiments for the application of budgeted influence maximization (BIM) on a portion of the Facebook network graph. BIM models the problem of identifying a low-cost subset (seed set) of nodes in a (known) social network that can influence the maximum number of nodes in a network. 

While there are prior works proposing algorithms for budgeted online influence maximization problems, the state of the art (e.g., \citep{perrault2020budgeted}) presumes knowledge of the diffusion model (such as independent cascade) and, more importantly, extensive semi-bandit feedback on individual diffusions, such as which specific nodes became active or along which edges successful infections occurred, in order to estimate diffusion parameters. For social networks with user privacy, this information is not available. 

\textbf{Data Set Description and Experiment Details:} The Facebook network dataset was introduced in \citep{NIPS2012_7a614fd0}. To facilitate running multiple experiments for different horizons, we used the community detection method proposed by \cite{Blondel2008FastUO} to detect a community with 354 nodes and 2853 edges. We further changed the network to be directed by replacing every undirected edge by two directed edges with opposite directions, yielding a directed network with 5706 edges. The diffusion process is simulated using the independent cascade model \citep{kempe2003maximizing}, where in each discrete step, an active node (that was inactive at the previous time step) independently attempts to infect each of its inactive neighbors. Following existing work of \citet{tang2015influence, tang2018online, Bian2020Efficient}, we set the probability of each edge $(u,v)$  as $1/d_{\mathrm{in}}(v)$, where $d_{\mathrm{in}}(v)$ is the in-degree of node $v$. %
In our experiment, we only consider users with high out-degrees as potential seeds, to spend our budget more efficiently. We pick the users with out-degrees that are above $95^{\mathrm{th}}$ percentile (18 users). Denote this set as $\mathcal{I}$, then for a user $u\in \mathcal{I}$, the cost is defined as %
    $c(u)=0.01d_{\mathrm{out}}(u) +1,$ %
similar to \citep{wu2022budgeted}. For each time horizon that was used, we ran each method ten times.

For this set of experiments, instead of cumulative $\frac{1}{2}$-regret, which requires knowing $\mathrm{OPT}$, we empirically compare the cumulative rewards achieved by C-ETC and OG$^\text{o}$ against $Tf(S^\mathrm{grd})$, where $S^\mathrm{grd}$ is the solution returned by the offline $\frac{1}{2}$-approximation algorithm proposed by \cite{yaroslavtsev2020bring}. $Tf(S^\mathrm{grd})\geq \frac{1}{2}Tf(\mathrm{OPT})$, so $Tf(S^\mathrm{grd})$ is a more challenging reference value.

\textbf{Results and Discussion: } \cref{fig:im:a,fig:im:b} show average cumulative regret curves for \etckgk\ (in blue), \etckga\ (in orange) and OG$^\text{o}$ (in green) for different horizon $T$ values when the budget constraint $B$ is 6 and 8, respectively. For $B=8$, the turning point is $T=21544$. Standard errors of means are presented as error bars, but might be too small to be noticed. \cref{fig:im:c,fig:im:d} are the instantaneous reward plots. The peaks at the very beginning of exploration phase correspond to the time step that the single person with highest influence is sampled. 

C-ETC significantly outperforms OG$^\text{o}$ for all time horizons and budget considered. To evaluate the gap between the empirical performance and the theoretical guarantee, we estimated the slope for both methods on log-log scale plots. Over the horizons tested,  OG$^\text{o}$'s cumulative regret (averaged over ten runs) has a growth rate of $0.98$. The growth rates of \etckgk\ for budgets 6 and 8 are $0.76$ and $0.68$, respectively. The growth rates of \etckga\ for budgets 6 and 8 are $0.75$ and $0.69$, respectively. The slopes are close to the $2/3 \approx 0.67$ theoretical guarantee, and notably, the performance for larger $B$ is better.

\subsection{Experiments with Song Recommendation}\label{apdx_song}
We then test our methods on the application of song recommendation on the Million Song Dataset \cite{Bertin-Mahieux2011}. In this problem, the agents aims to recommend a bundle of songs to users such that they are liked by as many users as possible. 

\begin{figure}[t]%
    \centering
    \subfloat[]{\label{fig:sr:a}{\includegraphics[width=0.35\linewidth]{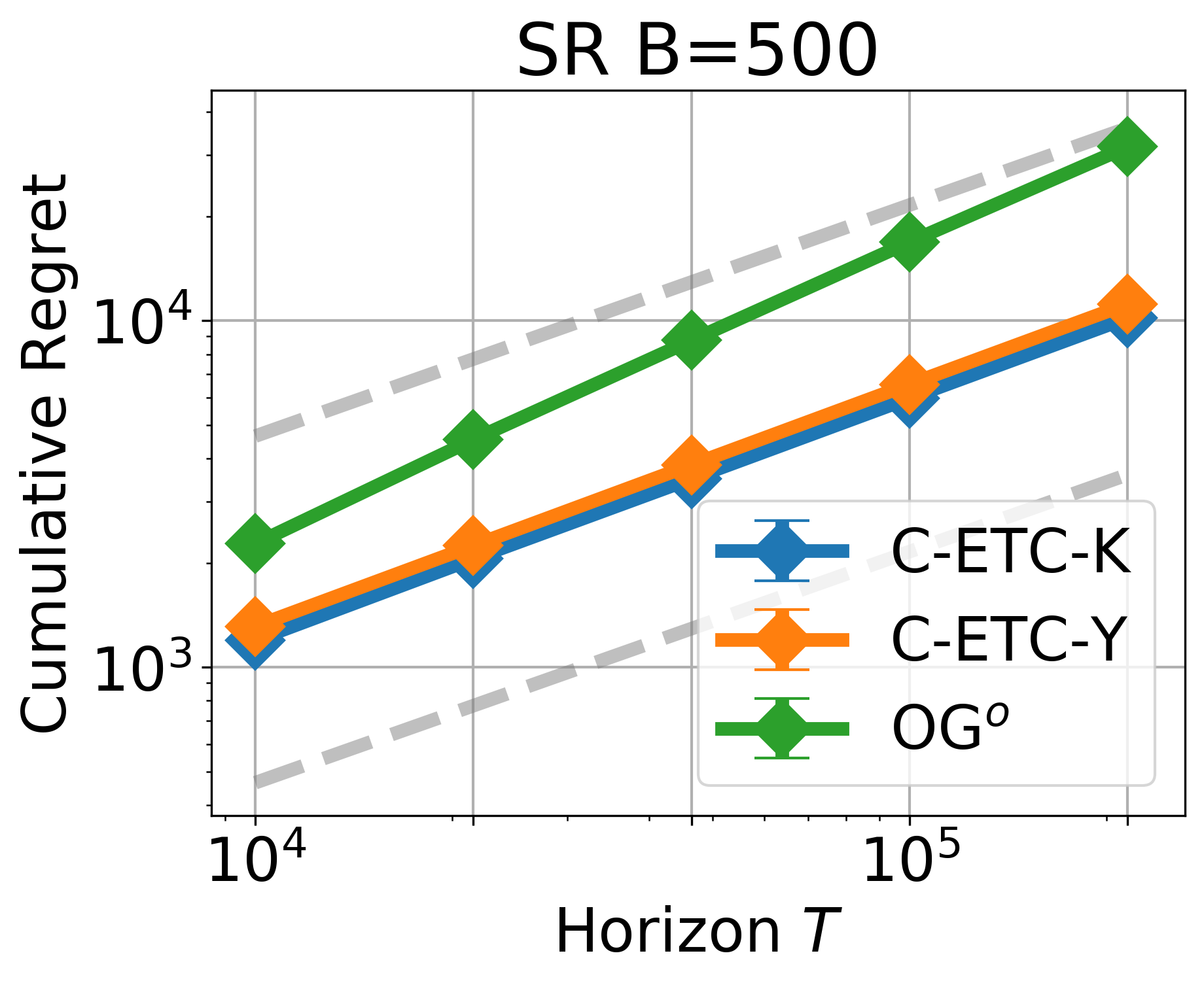} }}\qquad%
    \subfloat[]{\label{fig:sr:b}{\includegraphics[width=0.35\linewidth]{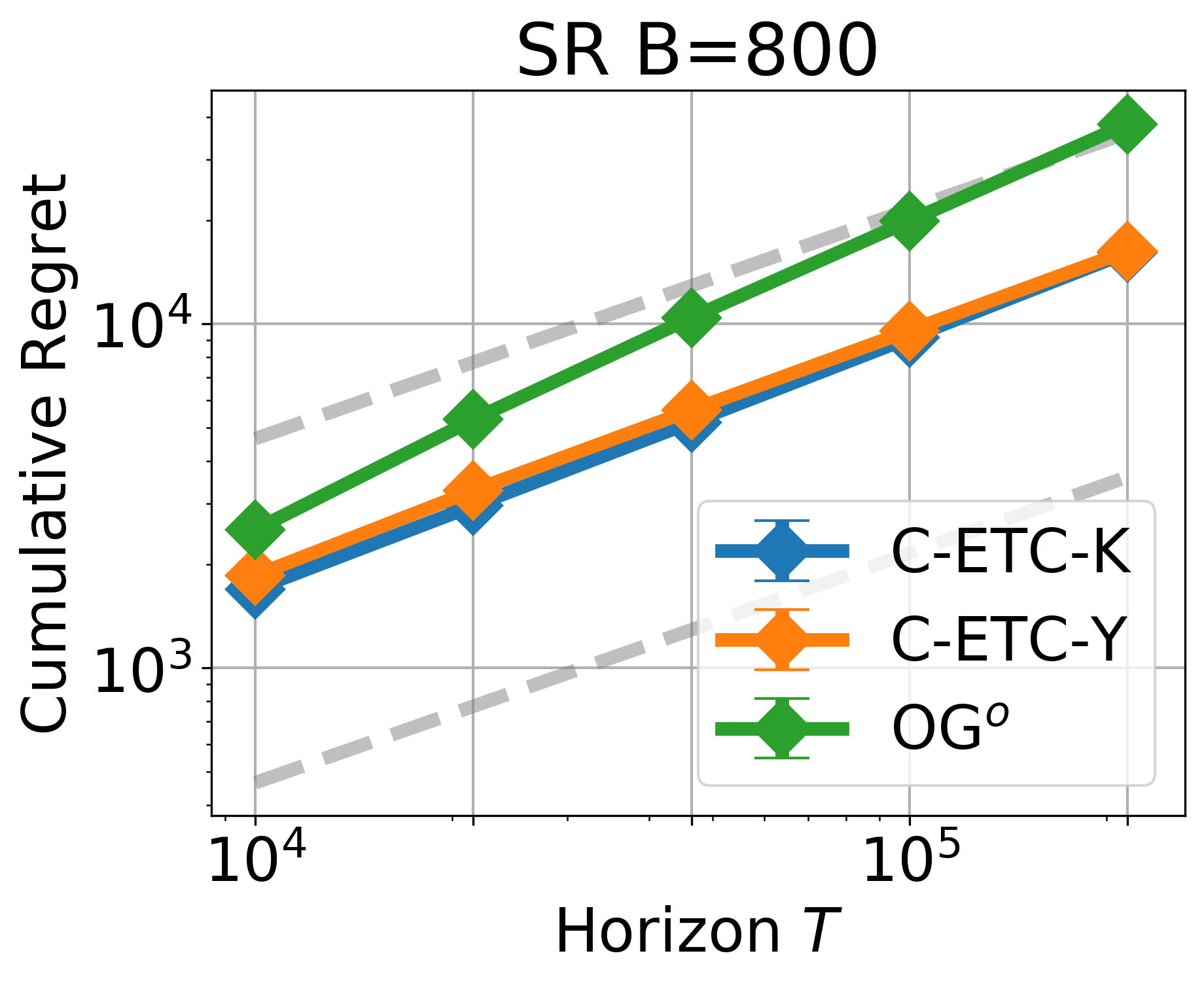} }}%
    \vspace{.5cm}\\
    \subfloat[]{\label{fig:sr:c}{\includegraphics[width=0.35\linewidth]{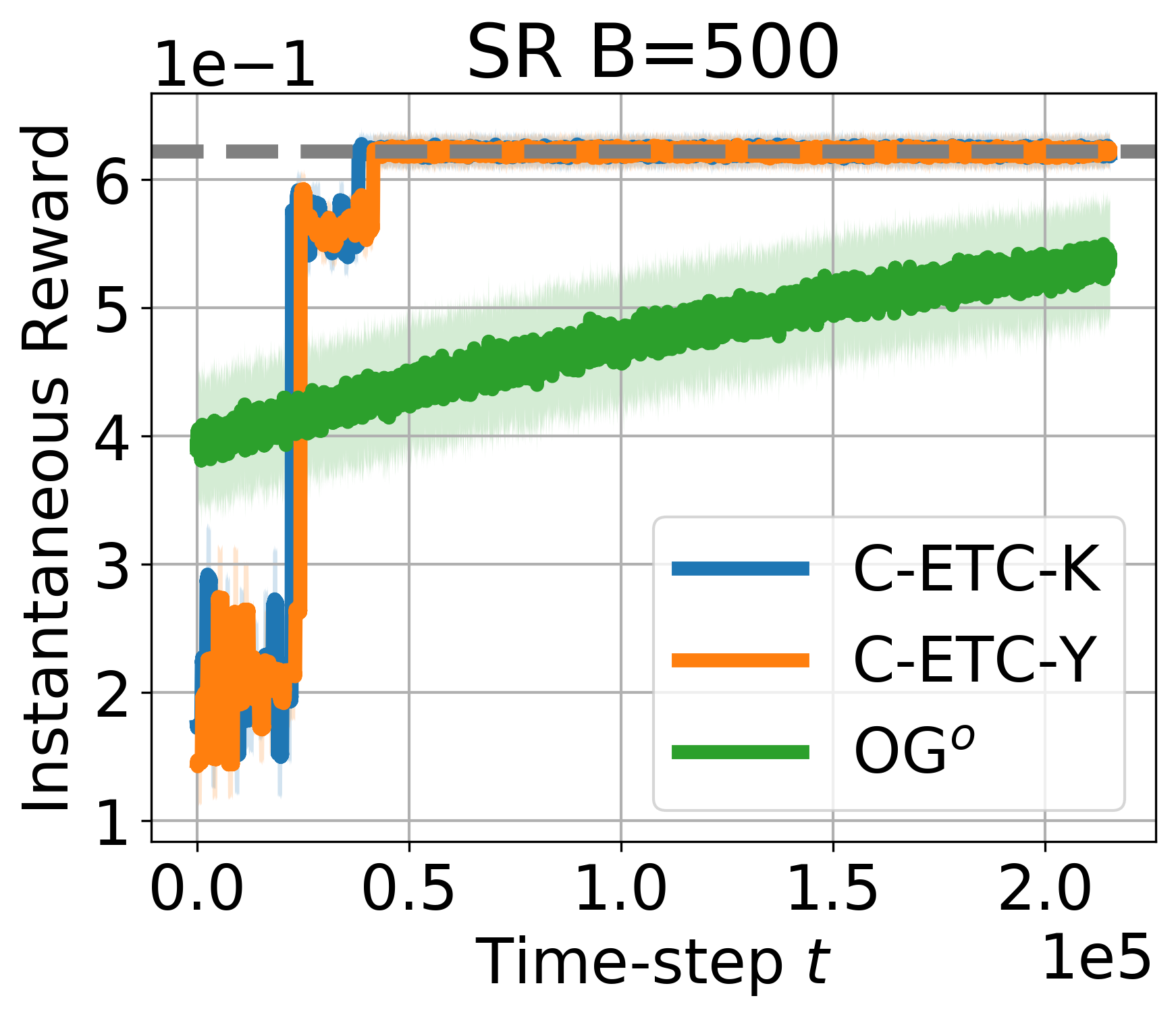} }}\qquad%
    \subfloat[]{\label{fig:sr:d}{\includegraphics[width=0.35\linewidth]{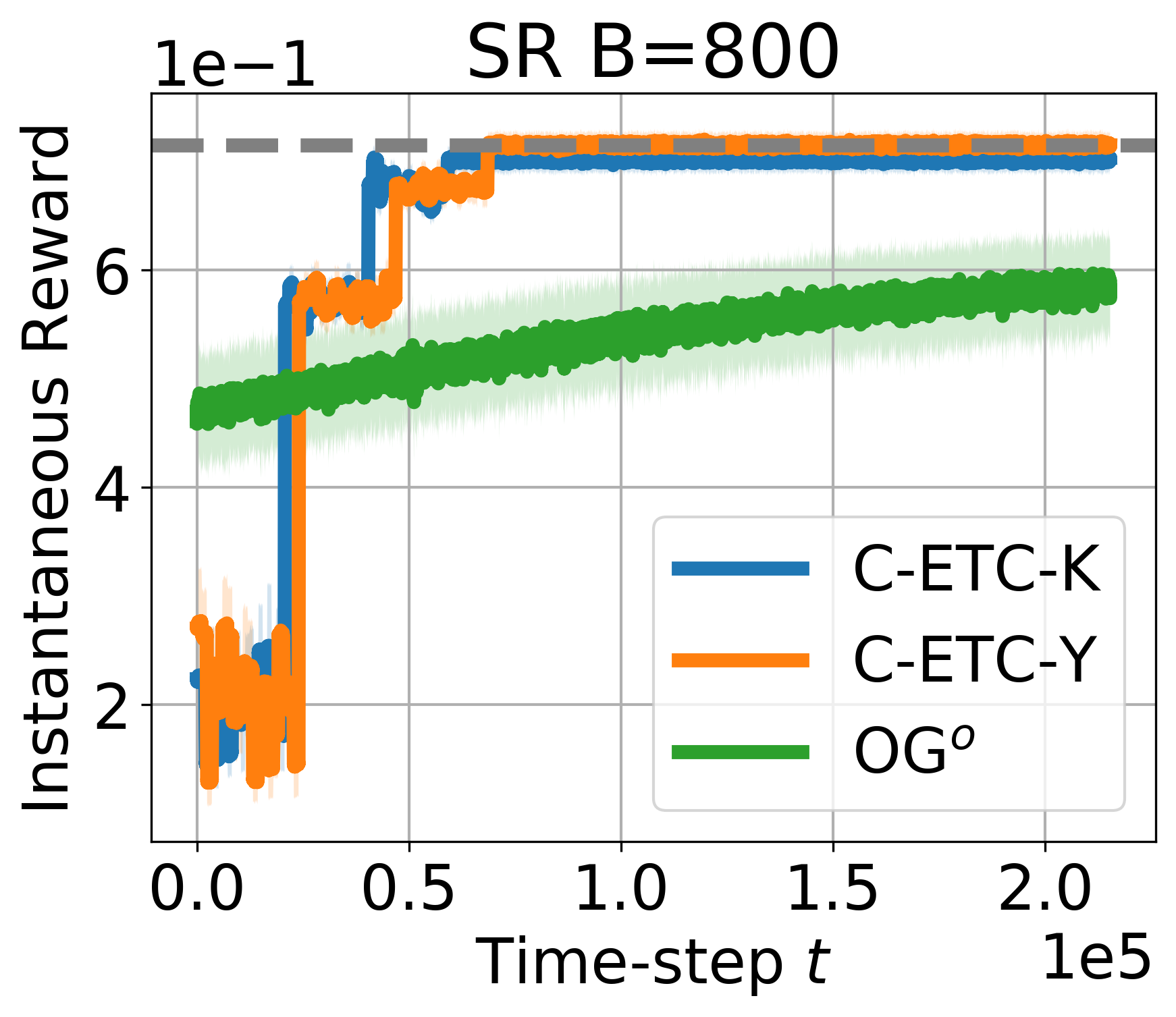} }}%
    \caption{Plots for song recommendation example. (a) and (b) are  results for cumulative regret as a function of time horizon $T$ for budgets $B=500$ and $B=800$ respectively. (c) and (d) are  instantaneous reward plots (smoothed with moving averages using a window size of 100) as a function of $t$ for budgets $B=500$ and $B=800$ respectively. The gray dashed lines in (a) and (b) represent $y = aT^{2/3}$ for various values of $a$ for visual reference.  The gray dashed lines in  (c) and (d) represent expected rewards for the action chosen by an offline greedy algorithm.}%
    \label{fig:sr}%
\end{figure}

\textbf{Data Set Description and Experiment Details}
From the Million Song Dataset \cite{Bertin-Mahieux2011}, we extract most popular 20 songs and 100 most active users. 

As in \citep{yue2011linear}, we model the system as  having a set of topics (or genres) $\mathcal{G}$ with $|\mathcal{G}|=d$ and for each item $e \in \Omega$, there is a feature vector $x(e):=\left(P_g(e)\right)_{g \in \mathcal{G}} \in \mathbb{R}^d$ that represents the information coverage on different genres. For each genre $g$, we define the probabilistic coverage function $f_g(S)$ by $1-\prod_{e \in S}\left(1-P_g(e)\right)$ and define the reward function $f(S)=\sum_i w_i f_i(S)$ with linear coefficients $w_i$. The vector $w:=$ $\left[w_1, \ldots, w_d\right]$ represents user preference on genres. %
In calculating $P_g(e)$ and $w$, we use the same formula for calculating $\bar{w}(e, g)$ and $\theta^*$ in \citep{pmlr-v115-hiranandani20a}. Like \citep{pmlr-v124-takemori20a}, we  define the cost of a song by its length (in seconds). For each user, the stochastic rewards of set $S$ are sampled from a Bernoulli distribution with parameter $f(S)$. For the total reward, we take the average over all users. When making the plots, we use statistics taken from 10 runs.

\textbf{Results and Discussion}
\cref{fig:sr:a,fig:sr:b} show average cumulative regret curves for \etckgk\ (in blue), \etckga\ (in orange) and OG$^\text{o}$ (in green) for different horizon $T$ values when the budget constraint $B$ is 500 and 800, respectively. \cref{fig:sr:c,fig:sr:d} are the instantaneous reward plots over a single horizon $T= 215,443$. Again, C-ETC significantly outperforms OG$^\text{o}$ for all time horizons and budget considered. We again estimated the slopes for both methods on log-log scale plots. Over the horizons tested,  OG$^\text{o}$'s cumulative regret (averaged over ten runs) has a growth rate above $0.85$. The growth rates of \etckgk\ for budgets 500 and 800 are $0.70$ and $0.73$, respectively. The growth rates of \etckga\ for budgets 500 and 800 are $0.70$ and $0.71$, respectively.

\section{Conclusions}\label{sec:concl}

This paper provides a general framework for adapting discrete offline approximation algorithms into sublinear $\alpha$-regret methods for stochastic CMAB problems where only bandit feedback is available with provable regret guarantees. %
We applied this framework to diverse applications in submodular maximization.  Exploring the applications in other domains where this general framework can be used is an interesting future direction.

\begin{commentediting3}
\clearpage
\end{commentediting3}

\bibliography{refs.bib}
\newpage
\appendix
\onecolumn
\section{Proof for Regret of C-ETC}
In this section, we prove \cref{thm:main} in \cref{sec:robust} of the main paper. We restate the theorem:
For the sequential decision making problem defined in Section 2 and $T\geq \frac{2\sqrt{2}N}{\delta}$, the expected cumulative $\alpha$-regret of C-ETC using an $(\alpha,\delta)$-robust approximation algorithm as subroutine is at most $\mathcal{O}\left(\delta^\frac{2}{3}N^\frac{1}{3} T^\frac{2}{3}\log(T)^\frac{1}{3}\right)$, where $N$ upper-bounds the number of value oracle queries made by the offline algorithm $\mathcal{A}$.

\subsection{Overview and Notations}

We will separate the proof into two cases.  The first case is for when the clean event $\mathcal{E}$  
happens, which we will show in \cref{lem:probcleanevents} happens with high probability.  Under the clean event, using the fact that the offline algorithm is an $(\alpha,\delta)$-robust approximation, C-ETC's chosen set $S$ for the exploitation phase will nonetheless be near-optimal. The second case is when the complementary event happens, which occurs with low probability.

The proof structure analyzing a high-probability ``clean event'' where empirical estimates are sufficiently concentrated around their means is analogous to that for the unstructured non-combinatorial setting (see for instance, Section 1.2 in \citep{MAL-068}). 
However, unlike the ETC procedure for non-combinatorial MAB problems, C-ETC makes sequences of decisions during exploration.  Furthermore, the combinatorial action space, non-linearity of the reward function, and lack of extra feedback (like marginal gains) make the problem challenging.  Even in the special setting of deterministic rewards, the standard MAB problem becomes trivial (finding the largest of $n$ base arms) while the problem we considered are NP-hard.

Recap that for any (feasible) action $A$, $f_t(A)$ denotes a (random) reward at time $t$ for the agent taking that action,   $f(A)$ denotes the expected value for action $A$. Let $\bar{f}_t(A)$ denote the empirical mean of rewards received from playing action $A$ up to and including time $t$. In the following, we will drop the subscript $t$ from the empirical mean, writing $\bar{f}(A)$ when it is clear from context that action $A$ has been played $m$ times. Also, we use $A_i$, $i\in \{1,\cdots, N\}$ denotes the i-th action the algorithm samples. We further denote $T_i, i\in \{1,\dots,N\}$ as the time step when the sampling of the i-th action has been determined, or $A_i$ has been played $m$ times. For notation consistency, we also denote $T_0=0$ and $T_{N+1}=T$.

\subsection{Probability of the Clean Event}
\label{sec:appd:proof:clean-event}
Now we define events that are important in our analysis. Recall that for each action $A$ being explored, the $m$ rewards  are i.i.d. with mean $f(A)$ and bounded in $[0,1]$. Thus, we can bound the deviation of the (unbiased) empirical mean $\bar{f}(A_i)$ from the expected value $f(A_i)$ for each action played.
Specifically, we can use a two-sided Hoeffding bound for bounded variables.

\begin{remark}
For convenience, we assume the reward function bounded in $[0,1]$, but the result can be generalized to the case where the deviation of the true reward and the expected reward has a light tailed distribution (e.g., sub-Gaussian). 
\end{remark}
\begin{lemma} [Hoeffding's inequality]
\label{lem:hoeffding}
Let $X_1, \cdots, X_n$ be independent random variables bounded in the interval $[0, 1]$, and let $\bar{X}$ denote their empirical mean. Then we have for any $\epsilon >0$,
\begin{align}
    \mathbb{P}\left( \big|\bar{X} -  \mathbb{E}[\bar{X}] \big| \geq \epsilon  \right) \leq 2 \mathrm{exp} \left( - 2 n \epsilon^2  \right). 
\end{align}
\end{lemma}

By C-ETC, each sampled action will be played the same number of times, denoted by $m$, so we consider bounding the probabilities of equal-sized confidence radii $\mathrm{rad} := \sqrt{\log(T)/2m}$ for all the actions played during exploration.

We next analyze the probability of the event that the empirical means of all actions played during exploration are concentrated around their statistical means within a radius $\mathrm{rad}$. Denote the corresponding events for each action played having empirical means concentrated around their respective statistical means as $\mathcal{E}_{i}$, 
\begin{align}
    \mathcal{E}_{i} &:=\bigcap\{\big|\bar{f}(A_i)-f(A_i) \big|< \mathrm{rad}\},\quad  
    i \in \{1,\cdots,N\}. \label{eq:iteration_event}
\end{align}
Define the \textit{clean event} $\mathcal{E}$ to be the event that the empirical means of all actions played in the exploration phase are within $\mathrm{rad}$ of their corresponding statistical means:
\begin{align}
    \mathcal{E} := \mathcal{E}_{1}\cap \dots \cap \mathcal{E}_{N}. \label{eq:clean_event}
\end{align}

\begin{lemma} \label{lem:probcleanevents}
The probability of the clean event $\mathcal{E}$ \eqref{eq:clean_event} satisfies:
\begin{align}
    \mathbb{P}(\mathcal{E}) %
    & \geq 1 - \frac{2N}{T}. \nonumber
\end{align}
\end{lemma}
\begin{proof}
Applying the Hoeffding bound \cref{lem:hoeffding} to  the empirical mean $\bar{f}(A_i)$ of $m$ rewards for action $A_i$  and choosing $\epsilon=\mathrm{rad}=\sqrt{\log(T)/2m}$ gives 

\begin{align}
    \mathbb{P} (\bar{\mathcal{E}}_i)&=\mathbb{P}\left[\big|\bar{f}(A_i)-f(A_i) \big| \geq \mathrm{rad} \right] \nonumber\\
    &\leq 2 \mathrm{exp} \left( - 2 m \mathrm{rad}^2  \right) \nonumber\\
    &= 2 \mathrm{exp} \left( - 2 m (\log(T)/2m ) \right) \nonumber\\
    &= 2 \mathrm{exp} \left( -  \log(T)  \right) \nonumber\\
    &= \frac{2}{T}. \label{eq:probbnd:single}
\end{align}
Then, we can bound the probability of clean events
\begin{align}
    \mathbb{P}(\mathcal{E}) &= \mathbb{P}(\mathcal{E}_1\cap \dots \cap \mathcal{E}_{N}) \nonumber\\
    &=1-\mathbb{P}(\bar{\mathcal{E}}_1\cup \dots \cup \bar{\mathcal{E}}_{N}) \tag{De Morgan's Law}\\
    &\geq 1-\sum_{i=1}^{N} \mathbb{P}(\bar{\mathcal{E}}_{i})\tag{union bounds} \\
    &\geq 1-\frac{2 N}{T}.  \tag{using \eqref{eq:probbnd:single}}
\end{align}
\end{proof}

\subsection{Near Optimality of the final $S$ (Exploitation Phase Action)}\label{sec:appd:proof:greedy-marginal-gains}

In \cref{lem:probcleanevents}, we showed that the clean event $\mathcal{E}$ will happen with high probability. When the clean event $\mathcal{E}$ happens, we have $|\bar{f}(A)-f(A)|\leq \mathrm{rad}$ for all evaluated action $A$. For an online algorithm (with output $S$) using an $(\alpha, \delta)$-robust approximation as subroutine, we have
\begin{align}
    \mathbb{E}[f(S)]\geq \alpha f(\mathrm{OPT})-\delta \cdot \mathrm{rad}. \label{eq:S:optimality}
\end{align}

\subsection{Final Regret}\label{sec:apdx:prf:main-thm}

Now we are ready to show the regret of C-ETC (\cref{thm:main} in \cref{sec:robust} of the main paper).

\subsubsection*{Case 1: clean event $\mathcal{E}$ happens}

In the first case we analyse the expected regret under the condition that the clean event $\mathcal{E}$ happens. In this section, all expectations will be conditioned on $\mathcal{E}$, but to simplify notation we will write $\mathbb{E}[\cdot]$ instead of $\mathbb{E}[\cdot|\mathcal{E}]$ in some cases.

First we can break up the expected $\alpha$-regret conditioned on $\mathcal{E}$ into two parts, one for the first $L$ exploration iterations, and the second for the exploitation iteration. Although the number of actions taken per iteration and the number of iterations of the greedy is not known a priori, we can upper bound the duration.   Also recall that $f_t(A_t)$ is the random reward for taking action $A_t$, which itself is random, depending on empirical means of actions in earlier iterations.  
\begin{align}
    \mathbb{E}[\mathcal{R}(T)|\mathcal{E}] 
    & = \alpha T f( \mathrm{OPT} ) -\sum_{t=1}^T \mathbb{E}[f_t(A_t)] \nonumber\\
    & = \alpha T f( \mathrm{OPT} ) -\sum_{t=1}^T \mathbb{E}[   \mathbb{E}[ f_t(A_t)  |A_t  ]] \tag{law of total expectation}\\
    & = \alpha T f( \mathrm{OPT} ) -\sum_{t=1}^T \mathbb{E}[f(A_t)] \tag{$f(\cdot)$ defined as expected reward}\\
    & = \sum_{t=1}^T \left(\alpha f(\mathrm{OPT})-\mathbb{E}[f(A_t)]\right) \tag{rearranging} \\
    & = \underbrace{\sum_{i=1}^{N} m \left(\alpha f(\mathrm{OPT})-\mathbb{E}[f(A_i)]\right)}_{\text{Exploration phase}} + \underbrace{\sum_{t=T_{N}+1}^T \left(\alpha f(\mathrm{OPT})-\mathbb{E}[f(A_t)]\right)}_{\text{Exploitation phase}} %
    \nonumber\\
    &=\sum_{i=1}^{N} m \left(\alpha f(\mathrm{OPT})-\mathbb{E}[f(A_i)]\right)+\sum_{t=T_{N}+1}^T \left(\alpha f(\mathrm{OPT})-\mathbb{E}[f(S)]\right). \label{eq:prf:mainthm:case1:60}
\end{align}

\paragraph{Case 1 (clean event): Bounding exploration regret:} We will separately bound the regret incurred from the exploration and exploitation.  We begin with bounding regret from exploration,

\begin{align}
    & \hspace{-2cm} \sum_{i=1}^{N} m \left(\alpha f(\mathrm{OPT})-\mathbb{E}[f(A_i)]\right) \nonumber \\ %
    &\leq  \sum_{i=1}^{N} m \left(\alpha-0\right) \tag{rewards are bounded in $[0,1]$}\\
    &\leq Nm. \label{eq:cum-regr-explor}
\end{align}

\paragraph{Case 1 (clean event): Bounding exploitation regret:}
We next bound the regret incurred during the exploitation iteration.  Since the set $S$ used during exploitation is a random variable, we can take the expectation of \eqref{eq:S:optimality} (conditioned on event $\mathcal{E}$), to bound the expected instantaneous regret for each time step of the exploitation iteration, 
\begin{align}
    \alpha f(\mathrm{OPT})-\mathbb{E}[f(S)]\leq \delta \mathrm{rad}. \label{eq:final_exp_reward}
\end{align}
Using a loose bound for the duration of the exploitation iteration, $ T-T_{L}+1 < T $,
\begin{align}
    \sum_{t=T_{N}+1}^T \left(\alpha f(\mathrm{OPT})-\mathbb{E}[f(S)]\right) %
    & \leq \sum_{t=T_{N}+1}^T \delta \mathrm{rad}  \tag{using \eqref{eq:final_exp_reward}} \nonumber \\
    & \leq T\delta \mathrm{rad}. \label{eq:cum-regr-exploit}
\end{align}
\paragraph{Case 1 (clean event): Bounding total regret:}
Then the expected cumulative regret \eqref{eq:prf:mainthm:case1:60} can be bounded  as
\begin{align}
    \mathbb{E}[\mathcal{R}(T)|\mathcal{E}] 
    &= \sum_{i=1}^{N} m \left(\alpha f(\mathrm{OPT})-\mathbb{E}[f(A_i)]\right)+\sum_{t=T_{N}+1}^T \left(\alpha f(\mathrm{OPT})-\mathbb{E}[f(S)]\right) \tag{copying \eqref{eq:prf:mainthm:case1:60}} \\
    & \leq Nm + T\delta \mathrm{rad} \tag{using \eqref{eq:cum-regr-explor}, \eqref{eq:cum-regr-exploit}}
\end{align}
Plugging in the formula for the confidence radius $\mathrm{rad} = \sqrt{\log(T)/2m}$, we have 

\begin{align}
    \mathbb{E}[\mathcal{R}(T)|\mathcal{E}] 
    & \leq Nm + T\delta \sqrt{\log(T)/2m} \nonumber \label{eq:final_regret1}
\end{align}
We want to optimize  $m$, the number of times each action is played.  Denoting the regret bound \eqref{eq:final_regret1} as a function of $m$ 
\begin{align}
    g(m) = Nm+ T\delta\sqrt{\log(T)/2m},
\end{align}
then 
\begin{align}
    g'(m) = N - \frac{1}{2}T\delta \sqrt{\log(T)/2} m^{-3/2}.
\end{align}
Setting $g'(m)=0$ and solving for $m$, %
\begin{align}
    \mopt= \frac{\delta^{2/3}T^{2/3}\log(T)^{1/3}}{2N^{2/3}}. \label{eq:value:m}
\end{align}
We next check the second derivative,
\begin{align}
    g''(m) = \frac{3}{4}\delta T\sqrt{\log(T)/2}m^{-5/2}. \label{eq:g2ndder}
\end{align}
For positive values of $m$, %
$g''(m) > 0$, thus $g(m)$ reaches a minimum at \eqref{eq:value:m}.
Since $m$ is the number of times actions are played, we (trivially) need $m\geq 1$ and $m$ to be an integer. We choose 
\begin{align}
    \moptrnd= \ceil*{\frac{\delta^{2/3}T^{2/3}\log(T)^{1/3}}{2N^{2/3}}}. \label{eq:int:value:m}
\end{align}
Since from \eqref{eq:g2ndder} we have that  $g''(m) >0$ for positive $m$, $g(\mopt)\leq g(\moptrnd)$. 
For $T\geq \frac{2\sqrt{2}N}{\delta}$, we have $\mopt \geq 1$. Further, to explore each action selected by the offline algorithm at least once, we need $T\geq N$. Overall, we require $T\geq \max\left\{N, \frac{2\sqrt{2}N}{\delta}\right\}$.

Plugging \eqref{eq:int:value:m} back in to \eqref{eq:final_regret1},  

\begin{align}
    \mathbb{E}[\mathcal{R}(T)|\mathcal{E}]
    &\leq \moptrnd N+ T\delta \sqrt{\log(T)/2\moptrnd} \tag{ \eqref{eq:final_regret1} with $\moptrnd$ samples for each action} \\
    &= \ceil*{\mopt} N+ T\delta \sqrt{\log(T)/2\ceil*{\mopt}} \nonumber\\
    &\leq  \ceil*{\mopt} N+ T\delta  \sqrt{\log(T)/2\mopt}  \tag{Since $\ceil{\mopt} \geq \mopt$} \\
    &\leq 2\mopt N+ T\delta \sqrt{\log(T)/2\mopt}  \tag{Since $\mopt \geq 1$, $\ceil{\mopt} \leq 2\mopt$} \\
    &= 2\frac{\delta^{2/3}T^{2/3}\log(T)^{1/3}}{2N^{2/3}} N \nonumber\\
    &\qquad + T\delta \sqrt{\log(T)/2}\left(\frac{\delta^{2/3}T^{2/3}\log(T)^{1/3}}{2N^{2/3}}\right)^{-1/2}\tag{using \eqref{eq:value:m}}\\
    & = 3\delta^{2/3}N^{1/3}T^{2/3}\log(T)^{1/3} \label{eq:final_regret_clean}\\
    &= \mathcal{O}\left(\delta^\frac{2}{3}N^\frac{1}{3} T^\frac{2}{3}\log(T)^\frac{1}{3}\right). \nonumber
\end{align} %
In conclusion, the expected $\alpha$-regret of C-ETC using an $(\alpha, \delta)$-robust approximation as subroutine is upper bounded by \eqref{eq:final_regret_clean} if the clean event $\mathcal{E}$ happens.

\subsubsection*{Case 2: clean event $\mathcal{E}$ does not happen}
We next derive an upper bound for the expected $\alpha$-regret for case that the event $\mathcal{E}$ does not happen. By \cref{lem:probcleanevents}, \[\mathbb{P}(\bar{\mathcal{E}}) = 1-\mathbb{P}(\mathcal{E}) \leq \frac{2N}{T}.\] Since the reward function $f_t(\cdot)$ is upper bounded by 1, the  expected $\alpha$-regret incurred under $\bar{\mathcal{E}}$ for a horizon of $T$ is at most $T$, %
\begin{align}
    \mathbb{E}[\mathcal{R}(T)|\bar{\mathcal{E}}] \leq T. \label{eq:badevent:regretbnd}
\end{align}

\subsubsection*{Putting it all together}

Combining Cases 1 and 2 we have, 
\begin{align}
    \mathbb{E}[\mathcal{R}(T)] &= \mathbb{E}[\mathcal{R}(T)|\mathcal{E}] \cdot \mathbb{P}(\mathcal{E}) +\mathbb{E}[\mathcal{R}(T)|\bar{\mathcal{E}}]\cdot \mathbb{P}(\bar{\mathcal{E}}) \tag{Law of total expectation}\\
    &\leq 3\delta^{2/3}N^{1/3}T^{2/3}\log(T)^{1/3}  \cdot 1 +T\cdot \frac{2N}{T} \tag{using \eqref{eq:final_regret_clean}, \cref{lem:probcleanevents}, and \eqref{eq:badevent:regretbnd}}\\
    &= \mathcal{O}\left(\delta^\frac{2}{3}N^\frac{1}{3} T^\frac{2}{3}\log(T)^\frac{1}{3}\right).
    \nonumber
\end{align}
This concludes the proof.

\section{Basic Facts} \label{apdx:facts}
\begin{fact} \label{fact:submod-basics:set-diff}
For a monotonically non-decreasing submodular set function $f$ defined over subsets of $\Omega$, we have for arbitrary subsets $A,B\subseteq \Omega$, 
\begin{align}
    f(B)-f(A) \leq \sum_{ j \in B \setminus A } \left[f(A\cup \{j\})-f(A)\right]. \nonumber
\end{align}
\end{fact}

\begin{fact} \citep{khuller1999budgeted}  \label{fact:khullerbnd}
For $x_1, \cdots, x_n \in \mathbb{R}^+$ such that $\sum x_i = A$, the function $[1-\prod_{i=1}^{n} (1-\frac{x_i}{A})]$ achieves its minimum at $x_1=x_2=\cdots=x_n=A/n$. 
\end{fact}

\begin{fact} \label{fact:submod-basics:1minus1ebnd} For $k\geq 1$, %
\begin{align}
    1- \left(1 - \frac{1}{k}\right)^k \geq 1-\frac{1}{e}. \nonumber
\end{align} 
\end{fact}

\section{Offline Approximation Algorithms for Submodular Maximization  -- Overview}\label{apdx_offlineover}
We  give a brief overview of the offline approximation algorithms  which we will analyze $(\alpha,\delta)$ robustness for. 

For a $k$-cardinality constraint, the  greedy algorithm \textsc{Greedy} proposed in \citet{nemhauser1978analysis} starts from an empty set $G\leftarrow \emptyset$. Then it repeatedly add the element with highest marginal gain $f(e|G)$ until the cardinality $|G|$ reaches $k$. \textsc{ThresholdGreedy}, proposed in \citet{badanidiyuru2014fast},  considers a sequence of decreasing thresholds: $\{\tau=d; \tau \geq \frac{\epsilon'}{n}d; \tau\leftarrow (1-\epsilon')\tau\}$ where $d=\max_{e\in \Omega}f(e)$. Then starting from empty set $G=\emptyset$, the algorithm includes any element $e \notin G$ such that $f(e|G)\geq \tau$ whenever the cardinality is smaller than $k$. %
The algorithm then repeats using a lower threshold.  %
\citet{badanidiyuru2014fast} showed that \textsc{ThresholdGreedy} can achieve $1-1/e-\epsilon'$ approximation.

For a knapsack constraint, several algorithms run the following greedy subroutine, which we refer to as \textsc{Greedy} (cardinality is a special case of this routine with budget $k$ and unit cost, so we keep the same name without confusion).  Start with empty set $G\gets\emptyset$.  Repeatedly add the element $e$ with the highest marginal density $\rho(e|G)$ that fits into the budget.  Let $G_i$ denote the set selected by \textsc{Greedy} that has cardinality $i$ 
and denote the constituent elements as $G_i=\{g_1,\cdots, g_i\}$. Let $L$ denote the cardinality of the final greedy set (i.e. when no more elements remain that are under budget), so $G_L$ is output by \textsc{Greedy}.  Note that $L$ can only be bounded ahead of time---there could be maximal subsets (to which no other elements could be added without violating the budget) of different cardinalities.

\textsc{Greedy} can have an unbounded approximation ratio \citet{khuller1999budgeted} for knapsack constraint. \citet{khuller1999budgeted} proposed \textsc{Greedy+}, which outputs the better of the best individual element $a^*\in \argmax_{e\in \Omega}f(e)$ and the output of \textsc{Greedy}, $\argmax_{S\in\{G_L,a^*\}} f(S)$.  \citet{khuller1999budgeted} proved that \textsc{Greedy+} achieves a $\frac{1}{2}(1-\frac{1}{e})$ approximation ratio. Then, \citet{Sviridenko2004ANO, khuller1999budgeted} proposed \textsc{PartialEnumeration}.  It first enumerate all sets with cardinality up to three. For each enumerated triplets, it build the rest of the solution set greedily. Then it outputs the set with largest value among all evaluated sets. They showed that \textsc{PartialEnumeration} can achieve $1-1/e$ approximation ratio. 

Greedy+Max generalizes \textsc{Greedy+} by augmenting each set $\{G_i\}_{i=1}^L$ in the nested sequence produced by \textsc{Greedy} with another element.  For $0\leq i \leq L-1$, define $G'_i \gets  G_i \cup \argmax_{e\in \Omega: c(G_i)+c(e)\leq B} f(G_i\cup e)$. By construction, $G'_0 = \{a^*\}$, the best individual element.  For $i=L$, $G'_L \gets G_L$.   \textsc{Greedy+Max} then outputs the best set in the augmented sequence, $\argmax_{S\in\{G'_0,\dots,G'_L\}} f(S)$.   \cite{yaroslavtsev2020bring} proposed \textsc{Greedy+Max} and proved it achieves an approximation ratio of $\frac{1}{2}$.  

A bound on the number of value oracle calls will be important in adapting offline methods. Denote $\beta := B/c_{\min}$ and $\Tilde{K} := \min \{n, \beta\}$ as an upper bound of the number of items in any feasible set.   We note here that while \textsc{PartialEnumeration} uses $\mathcal{O}(\Tilde{K}n^4)$ function evaluations, both \textsc{Greedy+Max} and \textsc{Greedy+} use $\mathcal{O}(\Tilde{K}n)$ oracle calls, same as \textsc{Greedy}. We use $N=\Tilde{K}n$ in the analysis for \textsc{Greedy+Max} and \textsc{Greedy+}.

\section{Proof for Robustness of Offline Algorithms for Submodular Maximization }\label{sec:offline:robust}

In this section, we prove the $(\alpha,\delta)$ robustness of algorithms considered in \cref{sec:appl-CETC:submod} of the main paper.
\subsection{Notation}
We first review notations used in the analysis. Recall that we are only able to evaluate the surrogate function $\hat{f}$ such that $|\hat{f}(S)-f(S)|\leq \epsilon$ for any feasible set $S$ and some $\epsilon > 0$, we further denote $\hat{f}(e|S)=\hat{f}(S\cup e)-\hat{f}(S)$ and $\hat{\rho}(e|S)=\frac{\hat{f}(S\cup e)-\hat{f}(S)}{c(e)}$. Let $G_i$ denote the set selected by basic \textsc{Greedy} (based on surrogate function $\hat{f}$) as described in Section 3 up until $i$th item and $G_i=\{g_1,\cdots, g_i\}$ in the order of each item is selected. 
Without loss of generality, define $G_0=\emptyset$ and $f(G_0)=\hat{f}(G_0)=0$. Denote $c_{\min} = \min_{e\in\Omega}c(e)$ be the item with lowest individual cost. Let $\beta = B/c_{\min}$ and $\Tilde{K} = \min \{n, \beta\}$ being an upper bound of the number of items in any feasible set. Since all selected actions should be feasible, for ease of notation, we omit denoting that condition throughout the proof. For example, we write $\argmax_{e \in \Omega \setminus A}f(e|A)$ to simplify the notation of $\argmax_{e: e \in \Omega \setminus A \text{ and } A\cup e\in D}f(e|A)$. Let $S$ be the set returned by modified algorithms in corresponding context.

\subsection{Robustness of Offline Methods for Submodular Maximization under Cardinality Constraint}
\label{sec:cardinality}
\subsubsection{\textsc{Greedy}}
We consider the original greedy algorithm \textsc{Greedy} proposed in \citet{nemhauser1978analysis}, which gives a $(1-\frac{1}{e})$-approximation guarantee for submodular maximization under a $k$-cardinality constraint. 
To restate \cref{thm:cardinality:greedy:robust} in the main paper, \textsc{Greedy} is a $(1-\frac{1}{e},2k)$-robust approximation algorithm for submodular maximization under a $k$-cardinality constraint.
The result follows from Corollary 4.3 of \citet{nie2022explore}, part of the regret analysis for a CMAB adaptation of \textsc{Greedy}.
\subsubsection{\textsc{ThresholdGreedy}}
We then consider the threshold greedy algorithm \textsc{ThresholdGreedy} proposed in \citet{badanidiyuru2014fast}, which gives a $(1-\frac{1}{e}-\epsilon')$-approximation guarantee for submodular maximization under a $k$-cardinality constraint, where $\epsilon'$ is a user specified parameter to balance accuracy and run time.   Restating \cref{thm:threshold:greedy:robust} in the main paper,
    \textsc{ThresholdGreedy} is a $(1-\frac{1}{e}-\epsilon',2(2-\epsilon')k)$-robust approximation algorithm for submodular maximization under a $k$-cardinality constraint.
\begin{proof}
    From the assumption of the surrogate function $\hat{f}$ we know
    \begin{align}
        f(e|S)-2\epsilon \leq \hat{f}(e|S) \leq f(e|S)+2\epsilon \nonumber
    \end{align}
    for any $e\in \Omega\setminus S$ and $S\subseteq \Omega$. Now assume the the next chosen element is $a$ and the current partial solution is $S$. On one hand, we have 
    \begin{align}
        \hat{f}(a|S)\geq w \Longrightarrow f(a|S) \geq w-2\epsilon, \label{eq:thg:1}
    \end{align}
    on the other hand, for every $e\in \mathrm{OPT}\setminus S$,
    \begin{align}
        \hat{f}(e|S) \leq \frac{w}{1-\epsilon'} \Longrightarrow f(e|S) \leq \frac{w}{1-\epsilon'}+2\epsilon. \label{eq:thg:2}
    \end{align}
    Combining and manipulating \eqref{eq:thg:1} and \eqref{eq:thg:2} we have for any $e\in \mathrm{OPT}\setminus S$:
    \begin{align}
        f(a|S) +2\epsilon \geq (f(e|S)-2\epsilon)(1-\epsilon')
        \Longrightarrow f(a|S)\geq (1-\epsilon')f(e|S) -2 (2-\epsilon')\epsilon.   
    \end{align}
    Taking an average over all $e\in \mathrm{OPT}\setminus S$,
    \begin{align}
        f(a|S) &\geq \frac{1-\epsilon'}{|\mathrm{OPT}\setminus S|}\sum_{e\in \mathrm{OPT}\setminus S}f(e|S) -2 (2-\epsilon')\epsilon \nonumber\\
        &\geq \frac{1-\epsilon'}{k}\sum_{e\in \mathrm{OPT}\setminus S}f(e|S) -2 (2-\epsilon')\epsilon. \label{eq:thg:3}
    \end{align}
    Now consider after $i\in[k-1]$ steps, we get a partial solution $S_i=\{a_1, \cdots, a_i\}$. By \eqref{eq:thg:3}, we have
    \begin{align}
        f(a_{i+1}|S_i) &\geq \frac{1-\epsilon'}{k}\sum_{e\in \mathrm{OPT}\setminus S}f(e|S_i) -2 (2-\epsilon')\epsilon \nonumber\\
        &\geq \frac{1-\epsilon'}{k}f( \mathrm{OPT}|S_i) -2 (2-\epsilon')\epsilon \tag{submodularity}\\
        &\geq \frac{1-\epsilon'}{k}\left(f(\mathrm{OPT})-f(S_i)\right) -2 (2-\epsilon')\epsilon, \tag{monotonicity}
    \end{align}
    and hence for $i\in [k-1]$, 
    \begin{align}
        f(S_{i+1})-f(S_i)=f(a_{i+1}|S_i) \geq \frac{1-\epsilon'}{k}\left(f(\mathrm{OPT})-f(S_i)\right) -2 (2-\epsilon')\epsilon. \label{eq:thg:4}
    \end{align}
    Using \eqref{eq:thg:4} as induction hypothesis, we then prove by induction (omitted) that for $i\in [k-1]$, 
    \begin{align}
        f(S_{i+1})\geq \left[1-\left(1-\frac{1-\epsilon'}{k}\right)^{i+1}\right]f(\mathrm{OPT})-2(i+1)(2-\epsilon')\epsilon, \nonumber
    \end{align}
    and plugging in $i=k-1$ we get
    \begin{align}
        f(S_k) &\geq \left[1-\left(1-\frac{1-\epsilon'}{k}\right)^k\right]f(\mathrm{OPT})-2k(2-\epsilon')\epsilon \nonumber\\
        &\geq (1-e^{-(1-\epsilon')})f(\mathrm{OPT})-2k(2-\epsilon')\epsilon \nonumber\\
        &\geq (1-1/e-\epsilon')f(\mathrm{OPT})-2k(2-\epsilon')\epsilon. \nonumber
    \end{align}
    We finish the proof by observing that $S_k$ is the output. 
\end{proof}

\subsection{Proof for Robustness of \textsc{Greedy+}} \label{sec:greedy+robust}
In this section, we prove \cref{thm:greedy:robust} in \cref{sec:appl-CETC:submod} of the main paper. The following statements, Lemmas~\ref{lem:consecutive_reward:apdx},\ref{lem:bnd-exp-val-aug-grd} and \ref{lem:sk_lower:apdx}, and their proofs are adapted from the proof of $\frac{1}{2}(1-\frac{1}{e})$ approximation ratio in the offline setting \cite{khuller1999budgeted} using a value oracle. %
\cite{Krause2005Anote} adapted the proof of \cite{khuller1999budgeted} to an offline setting where the greedy process relies on an exact oracle to evaluate individual element values and to compare the best individual element to the set output by the greedy process, but use an inexact value oracle (within $\epsilon$ of the correct value) to evaluate marginal densities.  %

The main  differences arise from (i) the algorithms of \cite{khuller1999budgeted,Krause2005Anote} evaluate densities before checking for feasibility,\footnote{As noted in \cref{foot:supp:evalclean},  concentration of estimates (i.e. the surrogate $\hat{f}$) used by C-ETC in the bandit setting will only be for evaluated subsets, which by restriction will all be feasible.} leading to different definitions of the augmented greedy sequence, necessitating us to use more care to show analogous properties, (ii) exact value oracles for best individual elements and for selecting OPT are used in  \cite{khuller1999budgeted,Krause2005Anote}, simplifying work to conclude the final bound for the approximation ratio $\alpha=\frac{1}{2}(1-\frac{1}{e})$ and leading to a different $\delta$. 

\begin{commentediting}
\color{blue}
we need to address the following, but maybe after submit main paper
\begin{itemize}\itemsep0ex
    \item this contains descriptions of 'expected marginal gains' - no randomness
    \item there's no discussion about overlap with khuller's proof or the note;maybe annotate after main deadline
    \item we can probably start with a common preliminary section for notation and bounds used in both robustness proofs.
    \item before submit supplementary file, think if can simplify case analysis
    \item naming of alg -- C-ETC shows up in proof robustness of GREEDY+
\end{itemize}
\color{black}
\end{commentediting}

Recall that \cref{thm:greedy:robust} in \cref{sec:appl-CETC:submod} of the main paper states that \textsc{Greedy+} is a $(\frac{1}{2}(1-\frac{1}{e}),2+\Tilde{K}+\beta)$-robust approximation algorithm for submodular maximization problem under a knapsack constraint.

We define $G_i$ and $g_i$ the same as previous section. Recall that the greedy process (using a surrogate $\hat{f}$) produces a nested sequence of subsets $\emptyset = G_0 \subset G_1 \subset \cdots \subset G_L$, where $L$ denotes the cardinality of the set final output of the greedy process. For the proof, we describe the greedy process as running for $L+1$ iterations, though on the final iteration no elements are added.

For any action $G_{i-1} \cup a$ evaluated in iteration $i$ of the greedy process, its marginal gains are upper bounded by that of the best subset based on surrogate function $\hat{f}$,

\begin{align}
    \frac{f(G_{i-1}\cup a)-f(G_{i-1})-2\epsilon}{c(a)} %
    &\leq  \frac {\hat{f}(G_{i-1}\cup a)-\hat{f}(G_{i-1})}{c(a)} \nonumber \\
    &\leq  \frac {\hat{f}(G_{i-1}\cup g_i)-\hat{f}(G_{i-1})}{c(g_i)} \tag{$g_i$ selected by greedy rule based on $\hat{f}$}\\
    &\leq  \frac{f(G_{i-1}\cup g_i)-f(G_{i-1})+2\epsilon}{c(g_i)} \nonumber\\
    &=  \frac{f(G_i)-f(G_{i-1})+2\epsilon}{c(g_i)}, \label{eq:clean_bound}
\end{align} %
where \eqref{eq:clean_bound} just uses the definition of $G_i \gets G_{i-1}\cup g_i$.  We will use \eqref{eq:clean_bound} to lower bound the true marginal gains (i.e. in terms of $f$) achieved for each iteration of the greedy process. 

Let $\ell \in\{1,\dots,L+1\}$ denote the first iteration for which there was an element $a'\in \Omega\backslash G_{\ell-1}$ whose cost exceeds the remaining budget ($c(a')+c(G_{\ell-1})>B)$ (thus subset $G_{\ell-1} \cup a'$ was not sampled), yet whose marginal density was higher than the marginal density of the chosen element $g_\ell$ up to $\pm 2 \epsilon$ normalized by the cost, specifically, for $\ell\leq L$,
\begin{align}
     \frac{f( G_{\ell-1} \cup a') - f( G_{\ell-1} ) - 2\epsilon}{c(a')} > \frac{f( G_{\ell-1} \cup a_\ell) - f( G_{\ell-1} ) + 2\epsilon}{c(a_r)}. \label{eq:prf:aug-greed:ineq}
\end{align}  If there is no such iteration $\ell<L+1$, then for $\ell= L+1$, we take the element $a'$ maximizing the term on the left hand side of  \eqref{eq:prf:aug-greed:ineq},
\begin{align}
    a' = \argmax_{a\in \Omega\backslash G_{\ell-1}} \  \frac{f( G_{\ell-1} \cup a) - f( G_{\ell-1} ) - 2\epsilon}{c(a)}. \label{eq:prf:aug-greed:ineq:2}
\end{align}
Likewise, if there is more than one element satisfying \eqref{eq:prf:aug-greed:ineq} for some (earliest) iteration $r$, then we also take the maximizer \eqref{eq:prf:aug-greed:ineq:2}.

We define an ``augmented'' greedy sequence of length $\ell$ which matches the greedy sequence up to the set of cardinality $\ell$, where the element $a'$ is selected despite violating the budget, 
\begin{align}
    \{\widetilde{G}_0=G_0=\emptyset, \widetilde{G}_1 = G_1, \dots,  \widetilde{G}_{\ell-1} = G_{\ell-1}, \widetilde{G}_\ell = G_{\ell-1}\cup\{a'\}  \} \label{eq:seq:aug-greedy}
\end{align} and correspondingly enumerate the elements of $\widetilde{G}_\ell$ in the order they were selected,
\begin{align}
    \{\widetilde{g}_1 = g_1, \dots, \widetilde{g}_{\ell-1} = g_{\ell-1}, \widetilde{g}_\ell = g'\}. \label{eq:seq:aug-greedy:arms}
\end{align}
We first prove the following lemma, bounding the marginal gains of the augmented greedy sequence $\{\widetilde{G}_0, \dots,  \widetilde{G}_\ell\}$.

\begin{lemma} \label{lem:consecutive_reward:apdx} 
For all   $i\in \{1,2,\cdots, \ell\}$, the following inequality holds:
\begin{align}
    f(\widetilde{G}_i)-f(\widetilde{G}_{i-1}) &\geq  \frac{c(\widetilde{g}_i)}{B}\left[f(\mathrm{OPT})-f(\widetilde{G}_{i-1})\right]-2 \left(1+\frac{\Tilde{K} c(\widetilde{g}_i)}{B}\right)\epsilon. \nonumber%
\end{align}
\end{lemma}  

\begin{proof}  Set any $i\in \{1,2,\cdots, \ell\}$. Let $\{v_1, v_2,. \cdots, v_k\} = \mathrm{OPT}\setminus \widetilde{G}_{i-1}$.  Note that by  construction \eqref{eq:seq:aug-greedy}, we have $\widetilde{G}_{i-1} = G_{i-1}$.

The difference $f(\mathrm{OPT})-f(\widetilde{G}_{i-1})$ can be bounded by the marginal gains of elements in the set difference,
\begin{align}
    f(\mathrm{OPT})-f(\widetilde{G}_{i-1}) %
    &\leq \sum_{j=1}^k \left[ f(\widetilde{G}_{i-1}\cup  v_j)-f(\widetilde{G}_{i-1}) \right] \tag{\cref{fact:submod-basics:set-diff}}\\ %
    &= \sum_{j=1}^k \left[ f(\widetilde{G}_{i-1}\cup  v_j)-f(\widetilde{G}_{i-1}) -2\epsilon +  2\epsilon \right] \nonumber \\ %
    &=  \sum_{j=1}^k c(v_j)\frac{f(\widetilde{G}_{i-1}\cup  v_j)-f(\widetilde{G}_{i-1})-2\epsilon}{c(v_j)} +  2k\epsilon\nonumber\\
    &\leq  \sum_{j=1}^k c(v_j) \frac{f(\widetilde{G}_{i-1}\cup \widetilde{g}_i)-f(\widetilde{G}_{i-1}) + 2\epsilon}{c(\widetilde{g}_i)} + 2k\epsilon  \label{eq:prf:consequtive_reward_main:10} \\
    &=  \sum_{j=1}^k c(v_j) \frac{f(\widetilde{G}_i)-f(\widetilde{G}_{i-1}) + 2\epsilon}{c(\widetilde{g}_i)} + 2k\epsilon \label{eq:prf:consequtive_reward_main:11}
\end{align} where \eqref{eq:prf:consequtive_reward_main:10} holds by following.  We consider four cases, depending on whether or not $\hat{f}(G_{i-1} \cup v_j)$ was evaluated during the iteration $i$.  

\begin{itemize}
    \item  \textbf{Case 1 ($\hat{f}(G_{i-1} \cup v_j)$ was evaluated and $i<\ell$):} At iteration $i$ (necessarily $i \leq L$ since no subsets were evaluated in iteration $L+1$) with current greedy set $G_{i-1}$, adding the element $v_j$ to the current greedy set was feasible, $c(v_j)\leq B-c(G_{i-1})$.  Then \textsc{Greedy+} would have evaluated $\hat{f}(G_{i-1} \cup v_j)$.  Since $v_j$ was not selected, the chosen element $g_i = G_i \backslash G_{i-1}$ must have had a higher surrogate density $\hat{f}(G_{i-1} \cup v_j)> \hat{f}(G_{i-1} \cup  g_i)$, so for $i<\ell$, for which $\widetilde{g}_i = g_i$ by construction \eqref{eq:seq:aug-greedy:arms}, \eqref{eq:clean_bound} implies \eqref{eq:prf:consequtive_reward_main:10}. 
    
    \item  \textbf{Case 2 ($\hat{f}(G_{i-1} \cup v_j)$ was evaluated and $i=\ell$):}  By the reasoning in the previous case, for the item $a_\ell$ chosen at iteration $\ell$ by the greedy process (due to feasibility and having the highest surrogate density), we still have the bound \eqref{eq:clean_bound} on true values, which coupled with our specific construction of $\widetilde{g}_\ell$ \eqref{eq:prf:aug-greed:ineq} means  
    
    \begin{align*}
       \frac{f(\widetilde{G}_{\ell-1}\cup  v_j)-f(\widetilde{G}_{\ell-1})-2\epsilon}{c(v_j)}
    &\leq   \frac{f(\widetilde{G}_{\ell-1}\cup a_r)-f(\widetilde{G}_{\ell-1}) + 2\epsilon}{c(a_r)}   \tag{by \eqref{eq:clean_bound}}\\
    &<   \frac{f(\widetilde{G}_{\ell-1}\cup \widetilde{g}_r)-f(\widetilde{G}_{\ell-1}) - 2\epsilon}{c(\widetilde{g}_r)}\tag{by construction \eqref{eq:prf:aug-greed:ineq}}\\
    &< \frac{f(\widetilde{G}_{\ell-1}\cup \widetilde{g}_r)-f(\widetilde{G}_{\ell-1}) + 2\epsilon}{c(\widetilde{g}_r)}.
    \end{align*}

    \item \textbf{Case 3 ($\hat{f}(G_{i-1} \cup v_j)$ was not evaluated and $i<\ell$):}  At iteration $i<\ell\leq L+1$ with the current greedy set $G_{i-1}$, adding the element $v_j$ to the current greedy set was not feasible, $c(v_j)> B-c(G_{i-1})$. By construction of the augmented greedy sequence, only at iteration $\ell$ was there an infeasible element whose surrogate marginal density satisfied the inequality \eqref{eq:prf:aug-greed:ineq}.  Thus, for iterations $i<\ell$, $G_{i-1}=\widetilde{G}_{i-1}$ and $G_i=\widetilde{G}_i$, so  \eqref{eq:prf:consequtive_reward_main:10} holds. 
    
    \item \textbf{Case 4 ($\hat{f}(G_{i-1} \cup v_j)$ was not evaluated and $i=\ell$):}  For iteration $i=\ell$, with current greedy set $G_{i-1}$, the augmented greedy sequence construction implies \eqref{eq:prf:consequtive_reward_main:10}.  Namely, with $i=\ell$,
    
        \begin{align*}
       \frac{f(\widetilde{G}_{\ell-1}\cup  v_j)-f(\widetilde{G}_{\ell-1})-2\epsilon}{c(v_j)}
    &<   \frac{f(\widetilde{G}_{\ell-1}\cup \widetilde{g}_r)-f(\widetilde{G}_{\ell-1}) - 2\epsilon}{c(\widetilde{g}_r)} \tag{by \eqref{eq:prf:aug-greed:ineq:2}}
    \\
    &< \frac{f(\widetilde{G}_{\ell-1}\cup \widetilde{g}_r)-f(\widetilde{G}_{\ell-1}) + 2\epsilon}{c(\widetilde{g}_r)}.
    \end{align*}
    menaing \eqref{eq:prf:consequtive_reward_main:10} holds.
    
\end{itemize}
We now continue lower bounding $f(\mathrm{OPT})-f(\widetilde{G}_{i-1})$,

\begin{align}
    f(\mathrm{OPT})-f(\widetilde{G}_{i-1}) 
    &\leq  \left[\sum_{j=1}^k c(v_j) \frac{f(\widetilde{G}_i)-f(\widetilde{G}_{i-1}) + 2\epsilon}{c(\widetilde{g}_i)}\right] + 2k\epsilon \tag{copying \eqref{eq:prf:consequtive_reward_main:11}}\\
    &= \left[\sum_{j=1}^k c(v_j)\right] \frac{f(\widetilde{G}_i)-f(\widetilde{G}_{i-1}) + 2\epsilon}{c(\widetilde{g}_i)} + 2k\epsilon \nonumber\\
    &\leq B \frac{f(\widetilde{G}_i)-f(\widetilde{G}_{i-1}) + 2\epsilon}{c(\widetilde{g}_i)} + 2k\epsilon \tag{$\mathrm{OPT}$ is feasible, so $\sum_{j=1}^k c(v_j) \leq B$}\\
    &\leq \frac{B}{c(\widetilde{g}_i)} \left[f(\widetilde{G}_i)-f(\widetilde{G}_{i-1})\right] + 2\left[\frac{B}{c(\widetilde{g}_i)} + \Tilde{K}\right]\epsilon  \tag{rearranging; $k\leq \Tilde{K}$}. \nonumber
\end{align}
Multiplying both sides by $\frac{c(\widetilde{g}_i)}{B}$ and rearranging finishes the proof.
\end{proof}

We unravel the recurrence in \cref{lem:consecutive_reward:apdx} to lower bound $f(\widetilde{G}_i)$.

\begin{lemma} \label{lem:bnd-exp-val-aug-grd}
For all $i\in \{1,2,\cdots, \ell\}$,
\begin{align}
    f(\widetilde{G}_i)  &\geq \left[1-\prod_{j=1}^i (1-\frac{c(\widetilde{g}_j)}{B})\right]f(\mathrm{OPT}) -2 (\beta + \Tilde{K})\epsilon. \nonumber
\end{align}
\end{lemma}

\begin{remark}
The steps to unravel the recurrence to obtain the first term (coefficient of $f(\mathrm{OPT})$) is the same as the proof for the analogous result in the offline setting \cite{khuller1999budgeted}.  The second term (with $\epsilon$) is due to working with marginal densities of a surrogate function $\hat{f}$.  The basic steps for working with that second term is the same as \cite{Krause2005Anote}, though we use a looser bound $\beta$; in \cite{Krause2005Anote} we think there may be a mistake in applying the induction step (with ``$c(X_i)$'' fixed for different $i$ in the proof), though they were loosely bounded with $\beta$ later on. 
\end{remark}

\begin{proof}  
The proof will follow by induction.  We first show the base case $i=1$ using  \cref{lem:consecutive_reward:apdx}.

\begin{align}
    f(\widetilde{G}_1) &= f(\widetilde{G}_1) - f(\widetilde{G}_0) \tag{$f$ is normalized; $\widetilde{G}_0 = \emptyset$}\\
    &\geq  \frac{c(\widetilde{g}_1)}{B}\left[f(\mathrm{OPT})-f(\widetilde{G}_0)\right]-2 \left(1+\frac{\Tilde{K} c(\widetilde{g}_1)}{B}\right)\epsilon \tag{using \cref{lem:consecutive_reward:apdx}}\\
    &=  \left[ 1- \left( 1-\frac{c(\widetilde{g}_1)}{B} \right)  \right]f(\mathrm{OPT})-2 \left(1+\frac{\Tilde{K} c(\widetilde{g}_1)}{B}\right) \epsilon\label{eq:prf:recurs:3}%
\end{align} where \eqref{eq:prf:recurs:3} follows from   rearranging.  For the second term in \eqref{eq:prf:recurs:3}, using that

\begin{align}
    1+\frac{\Tilde{K} c(\widetilde{g}_1)}{B} %
    &\leq \frac{B}{c(\widetilde{g}_1)}\left( 1+\frac{\Tilde{K} c(\widetilde{g}_1)}{B} \right) \tag{since $\frac{B}{c(\widetilde{g}_1)}\geq 1$}\\
    &= \frac{B}{c(\widetilde{g}_1)}+\Tilde{K} \nonumber\\
    &\leq \frac{B}{\cmin}+\Tilde{K}  \nonumber\\
    &= \beta+\Tilde{K}, \label{eq:prf:recurs:3b}
\end{align}
then \begin{align}
    f(\widetilde{G}_1) 
    &\geq  \left[ 1- \left( 1-\frac{c(\widetilde{g}_1)}{B} \right)  \right]f(\mathrm{OPT})-2 \left(1+\frac{\Tilde{K} c(\widetilde{g}_1)}{B}\right)\epsilon \tag{copying  \eqref{eq:prf:recurs:3}} \\
    &\geq  \left[ 1- \left( 1-\frac{c(\widetilde{g}_1)}{B} \right)  \right]f(\mathrm{OPT})-2 (\beta+\Tilde{K}) \epsilon\tag{using \eqref{eq:prf:recurs:3b} }.
\end{align}
This completes the base case of $i=1$.

We next consider $i>1$.  Unraveling the recurrence shown in \cref{lem:consecutive_reward:apdx},
\begin{align}
    f(\widetilde{G}_i) %
    &= f(\widetilde{G}_i)-f(\widetilde{G}_{i-1}) +f(\widetilde{G}_{i-1}) \nonumber\\ %
    &\geq \left[ \frac{c(\widetilde{g}_i)}{B}\left(f(\mathrm{OPT})-f(\widetilde{G}_{i-1})\right)-2 \left(1+\frac{\Tilde{K} c(\widetilde{g}_i)}{B}\right)\epsilon\right] +f(\widetilde{G}_{i-1}) \tag{using \cref{lem:consecutive_reward:apdx}}\\
    &=  \left[\frac{c(\widetilde{g}_i)}{B}\right]  f(\mathrm{OPT}) -2 \left(1+\frac{\Tilde{K} c(\widetilde{g}_i)}{B}\right)\epsilon + \left[1-\frac{c(\widetilde{g}_i)}{B}\right]f(\widetilde{G}_{i-1}) \tag{rearranging}\\
    &=  \left[1-(1-\frac{c(\widetilde{g}_i)}{B})\right]  f(\mathrm{OPT}) -2 \left(1+\frac{\Tilde{K} c(\widetilde{g}_i)}{B}\right)\epsilon \nonumber\\
    &\quad + \left[1-\frac{c(\widetilde{g}_i)}{B}\right]f(\widetilde{G}_{i-1}) \tag{rearranging}\\
    &\geq  \left[1-(1-\frac{c(\widetilde{g}_i)}{B})\right]  f(\mathrm{OPT}) -2 \left(1+\frac{\Tilde{K} c(\widetilde{g}_i)}{B}\right)\epsilon \nonumber\\
    &\quad + \left(1-\frac{c(\widetilde{g}_i)}{B}\right)\left[\left(1-\prod_{j=1}^{i-1} (1-\frac{c(\widetilde{g}_j)}{B})\right)f(\mathrm{OPT}) -2 (\beta + \Tilde{K})\epsilon\right] \tag{induction step}\\
    &= \left[1-(1-\frac{c(\widetilde{g}_i)}{B}) + \left(1-\frac{c(\widetilde{g}_i)}{B}\right)\left(1-\prod_{j=1}^{i-1} (1-\frac{c(\widetilde{g}_j)}{B})\right) \right]  f(\mathrm{OPT}) \nonumber\\
    &\quad  -2 \left(1+\frac{\Tilde{K} c(\widetilde{g}_i)}{B}
    +\left(1-\frac{c(\widetilde{g}_i)}{B}\right)(\beta + \Tilde{K})
    \right)\epsilon \tag{rearranging}\\
    &= \left[1-\prod_{j=1}^{i} (1-\frac{c(\widetilde{g}_j)}{B}) \right]  f(\mathrm{OPT}) \nonumber\\
    &\quad  -2 \left(1
    +\beta - \beta \frac{c(\widetilde{g}_i)}{B} + \Tilde{K} 
    \right)\epsilon. \label{eq:prf:recurs:8}%
\end{align}
For the second term in \eqref{eq:prf:recurs:8}, using that 
\begin{align}
    \beta \frac{c(\widetilde{g}_i)}{B} &= \frac{B}{\cmin}\frac{c(\widetilde{g}_i)}{B}\tag{def. of $\beta$}\\
    &=\frac{c(\widetilde{g}_i)}{\cmin} \nonumber\\
    &\geq 1, \label{eq:prf:bnd-err-recur:5}
\end{align}
then 
\begin{align}
    -2 \left(1
    +\beta - \beta \frac{c(\widetilde{g}_i)}{B} + \Tilde{K} 
    \right)\epsilon %
    &= -2 \left(\beta +\Tilde{K} \right)\epsilon + 2 \left( \beta \frac{c(\widetilde{g}_i)}{B} -1    \right)\epsilon \tag{rearranging}\\
     &\geq -2 \left(\beta +\Tilde{K} \right)\epsilon. \tag{using \eqref{eq:prf:bnd-err-recur:5}} \nonumber
\end{align}
Applying this to \eqref{eq:prf:recurs:8} completes the proof.
\end{proof}

The inequality in \cref{lem:bnd-exp-val-aug-grd} for the augmented greedy set of cardinality $\ell$ can be further simplified.  We will use the following observations.

\begin{lemma}
\label{lem:sk_lower:apdx}
The following inequality holds:
\begin{align}
    f(\widetilde{G}_\ell) & \geq (1-\frac{1}{e})f(\mathrm{OPT})-2(\beta+\Tilde{K})\epsilon. \nonumber
\end{align}
\end{lemma}

\begin{proof} %
Applying $i=\ell$ to \cref{lem:bnd-exp-val-aug-grd} and bounding the coefficient for $f(\mathrm{OPT})$,
\begin{align}
    f(\widetilde{G}_\ell)  %
    &\geq \left[1-\prod_{j=1}^\ell (1-\frac{c(\widetilde{g}_j)}{B})\right]f(\mathrm{OPT}) -2 (\beta + \Tilde{K})\epsilon \nonumber\\
    &\geq \left[1-\prod_{j=1}^\ell (1-\frac{c(\widetilde{g}_j)}{c(\widetilde{G}_\ell)})\right]f(\mathrm{OPT}) -2 (\beta + \Tilde{K})\epsilon \tag{by construction, $c(\widetilde{G}_\ell)>B$} \\
    &\geq \left[1-\prod_{j=1}^\ell (1-\frac{c(\widetilde{G}_\ell)/\ell}{c(\widetilde{G}_\ell)})\right]f(\mathrm{OPT}) -2 (\beta + \Tilde{K})\epsilon \tag{using \cref{fact:khullerbnd}} \\
    &= \left[1- (1-\frac{1}{\ell})^\ell\right]f(\mathrm{OPT}) -2 (\beta + \Tilde{K})\epsilon \tag{simplifying} \\
    &\geq \left(1-\frac{1}{e}\right)f(\mathrm{OPT}) -2 (\beta + \Tilde{K})\epsilon \tag{using \cref{fact:submod-basics:1minus1ebnd}}.
\end{align}  

\end{proof}

Using the aforementioned lemmas, we are now ready to complete the proof for Theorem 3 (robustness of \textsc{Greedy+} algorithm). We will bound the value of set $G_L$ using the results on the augmented greedy set \eqref{eq:seq:aug-greedy} of cardinality  $\ell$, and in turn bound the value of the set $S$, the final output of \textsc{Greedy+}.  

\begin{commentediting}
\begin{lemma}
\label{lem:bnd:exploitedset:apdx}
The final output set $S$ of \textsc{Greedy+} satisfies
\begin{align}
    f(S) %
    &\geq \frac{1}{2}  (1-\frac{1}{e})f(\mathrm{OPT})-(2+\beta+\Tilde{K})\epsilon. \nonumber
\end{align}
\end{lemma}

\end{commentediting}

Recall that \textsc{Greedy+} chooses the set $S$ to be either the best individual element (based on $\hat{f}$) $a^* \gets \argmax_{e\in \Omega} \hat{f}(e)$ or the output of the greedy process $G_L$.  Let $a^\mathrm{OPT} = \argmax_{e\in \Omega} f(e)$ denote the element with the highest value under $f$.  Then
\begin{align}
    f(a^*) &\geq \hat{f}(a^*) - \epsilon \nonumber\\
    &\geq \hat{f}(a^\mathrm{OPT}) - \epsilon \tag{by definition of $a^*$}\\
    &\geq f(a^\mathrm{OPT}) - 2\epsilon. \label{eq:prf:halfbnd:4}
\end{align}

By construction \eqref{eq:seq:aug-greedy}, $\widetilde{G}_\ell$ includes one more element $a'$ than $\widetilde{G}_{\ell-1}$  (and $a'$ maximizes \eqref{eq:prf:aug-greed:ineq:2}).  By submodularity, the marginal gain of $a'$ is bounded by $f(a')$ and in turn by the best individual element based on surrogate function $\hat{f}$,
\begin{align}
    f(\widetilde{G}_{\ell-1}) + f(a^\mathrm{OPT}) %
    &\geq f(\widetilde{G}_{\ell-1}) + f(a') \tag{by definition of $a^\mathrm{OPT}$}\\
    &\geq f(\widetilde{G}_{\ell-1}) + \left[f(\widetilde{G}_{\ell-1} \cup a') - f(\widetilde{G}_{\ell-1}) \right] \tag{by submodularity}\\
    &= f(\widetilde{G}_{\ell-1} \cup a')  \nonumber\\
    &= f(\widetilde{G}_\ell)  \tag{by construction \eqref{eq:seq:aug-greedy}} \\
    &\geq (1-\frac{1}{e})f(\mathrm{OPT})-2(\beta+\Tilde{K})\epsilon, \label{eq:prf:halfbnd:9}
\end{align} where \eqref{eq:prf:halfbnd:9} follows from \cref{lem:sk_lower:apdx}.

Also by construction \eqref{eq:seq:aug-greedy}, the greedy and augmented greedy processes match up to and including the set of cardinality $\ell-1$, so 
\begin{align}
    f(G_L)
    &\geq f(G_{\ell-1}) \tag{monotonicity}\\
    &= f(\widetilde{G}_{\ell-1}). \tag{By construction \eqref{eq:seq:aug-greedy}}
\end{align}

Thus, 
\begin{align}
    f(G_L) + f(a^\mathrm{OPT}) %
    &\geq f(\widetilde{G}_{\ell-1}) + f(a^\mathrm{OPT}) \nonumber\\%
    &\geq (1-\frac{1}{e})f(\mathrm{OPT})-2(\beta+\Tilde{K})\epsilon. \tag{using \eqref{eq:prf:halfbnd:9}}
\end{align} 

At least one of %
$f(G_L)$
and $f(a^\mathrm{OPT}) $ is at least half of the value of the right hand side, 
\begin{align}
    \max\{ f(G_L), f(a^\mathrm{OPT}) \} &\geq \frac{1}{2}(1-\frac{1}{e})f(\mathrm{OPT})-(\beta+\Tilde{K})\epsilon \label{eq:bnd:maxGL-aOPT}
\end{align}
Thus, for the chosen set $S$
\begin{align}
    f(S) &\geq \hat{f}(S) - \epsilon \nonumber\\
    &= \max\{ \hat{f}(G_L), \hat{f}(a^*) \} - \epsilon \nonumber\\
    &\geq \max\{ \hat{f}(G_L), \hat{f}(a^\mathrm{OPT}) \} - \epsilon \tag{$a^*$ is the element with largest $\hat{f}$ value}\nonumber\\
    &\geq \max\{ f(G_L)-\epsilon, f(a^\mathrm{OPT})-\epsilon \} - \epsilon \tag{element-wise dominance} \nonumber\\
    &= \max\{ f(G_L), f(a^\mathrm{OPT}) \} - 2\epsilon  \nonumber\\
    &\geq \frac{1}{2}(1-\frac{1}{e})f(\mathrm{OPT})-(\beta+\Tilde{K})\epsilon- 2\epsilon \tag{from \eqref{eq:bnd:maxGL-aOPT}}\nonumber\\
    &=  \frac{1}{2}(1-\frac{1}{e})f(\mathrm{OPT})-(2+\beta+\Tilde{K})\epsilon. \nonumber
\end{align}

which completes the proof.

\subsection{Proof for Robustness of \textsc{PartialEnumeration}} 
Now we analyze the \textsc{PartialEnumeration} algorithm for submodular maximization under a knapsack constraint proposed in \citet{Sviridenko2004ANO, khuller1999budgeted}.  Recall that \cref{thm:partial:enum:greedy:robust} in \cref{sec:appl-CETC:submod} of the main paper states 
    \textsc{PartialEnumeration} is a $(1-\frac{1}{e},4+2\Tilde{K}+2\beta)$-robust approximation algorithm for submodular maximization under a knapsack constraint.
\begin{proof}
Assume $|\mathrm{OPT}| > 3$, otherwise the algorithm finds a $(1, 2)$-robust approximation, so it is also a $(1-\frac{1}{e},2(\Tilde{K}+\beta))$-robust approximation for non-trivial cases where $\Tilde{K}\geq 1$ and $\beta\geq 1$. Enumerate the elements of the optimal solution as $\mathrm{OPT}=\{Y_1,\cdots, Y_m\}$, corresponding to the order they would be selected by the simple greedy algorithm (iteratively selecting the element with the largest marginal gain, not the largest marginal density) 
\begin{align}
    Y_{i+1}=\argmax_{Y\in \mathrm{OPT}}f(\{Y_1,\cdots, Y_i, Y\})-f(\{Y_1,\cdots, Y_i\}),
\end{align}  
and let $R=\{Y_1, Y_2, Y_3\}$. Consider the iteration where the algorithm considers $R$. Define the function 
\begin{align}
    f'(A)=f(A\cup R)-f(R).
\end{align}
$f'$ is a non-decreasing submodular set function with $f'(\emptyset)=0$, and the optimal solution (with budget $B-c(R)$) is $\mathrm{OPT}\setminus R$ since for any set $S$ with cost $c(S)\leq B-c(R)$,
\begin{align}
    f'(\mathrm{OPT}\setminus R) 
    &= f(\mathrm{OPT} \cup R )-f(R)\tag{def of $f'$}\\ %
    &= f(\mathrm{OPT})-f(R)\tag{$R \subseteq \mathrm{OPT}$ by construction}\\
    &\geq f(S\cup R)-f(R)\nonumber\\
    &= f'(S). \nonumber
\end{align}
Hence we can apply \textsc{Greedy+} algorithm to $f'$ (based on noisy evaluations). 
Let $g_\ell$ be the first element from $\mathrm{OPT} \setminus R$ which could not be added due to budget constraints,
and let $A=\{g_1, \cdots, g_{\ell-1}\}$ be first $\ell-1$ elements selected by \textsc{Greedy+} algorithm.
\begin{commentediting3}
\cjq{this is confusing.  Are you specifically considering those selected when run on the triplet $R$? or the outputted set of the 1-1/e algorithm?}.\ngy{the former. so after selecting $\ell-1$ elements there will be $\ell+2$ elements in total as a candidate solution} 
\end{commentediting3}
Let $G=A\cup R$. Using \cref{lem:sk_lower:apdx}, we get 
\begin{align}
    f'(A\cup g_\ell) & \geq (1-\frac{1}{e})f'(\mathrm{OPT}\setminus R)-2(\beta'+\Tilde{K}')\epsilon, \nonumber
\end{align}
where $\beta'=\frac{B-c(R)}{c_{\min}'}$, $\Tilde{K}'=\min\{n-3, \beta'\}$ and $c_{\min}'=\min_{e\in \Omega \setminus R}c(e)$. Simple calculation can show that $\beta'\leq \beta$ and $\Tilde{K}'\leq \Tilde{K}$. Thus,
\begin{align}
    f'(A\cup g_\ell) & \geq (1-\frac{1}{e})f'(\mathrm{OPT}\setminus R)-2(\beta+\Tilde{K})\epsilon, \nonumber
\end{align}
From the definition of $f'$, we have $f(G)=f'(A)+f(R)$. Let $\Delta = f'(A\cup g_\ell)-f'(A)$. We have
\begin{align}
    f'(A)+\Delta \geq (1-\frac{1}{e})f'(\mathrm{OPT}\setminus R)-2(\beta+\Tilde{K})\epsilon.
\end{align}
Further observe that elements in $\mathrm{OPT}$ are ordered that for all $1\leq i\leq 3$, 
\begin{commentediting3}
\ngy{stated here}\cjq{what is $i$ here?  is $i$ the index of $g_l$? or is it any index smaller than the index of $g_l$? or is it any of indices $1,2,3$ so you begin by considering B? o/w why does the inequality hold;  this is confusing. explicitly remind the reader that by construction, since $g_\ell \not \in B$, and $B=\{Y_1,Y_2,Y_3\}$, $g_\ell$ corresponds to $Y_j$ for some $j>3$.  Next consider any $i\leq 3<j$}    
\end{commentediting3}
\begin{align}
    &\hspace{-1cm}f(\{Y_1,\cdots, Y_i\})-f(\{Y_1,\cdots, Y_{i-1}\}) \nonumber\\
    \geq& f(\{Y_1,\cdots, Y_{i-1},g_\ell\})-f(\{Y_1,\cdots, Y_{i-1}\}) \tag{ordering rule}\\
    \geq& f(R\cup A\cup g_\ell)-f(R\cup A) \tag{$\{Y_1,\cdots, Y_{i-1}\} \subseteq R$ when $1\leq i\leq 3$ and submodularity}\\
    =&f(R\cup A\cup g_\ell) -f(R) -(f(R\cup A) -f(R))\nonumber\\
    =&f'(A\cup g_\ell)-f'(A)\nonumber\\
    =&\Delta. \nonumber
\end{align}
By telescoping sum, $f(R)\geq 3\Delta$. Now we get 
\begin{align}
    f(G) &= f(R)+f'(A) \nonumber\\
    &\geq f(R)+(1-\frac{1}{e})f'(\mathrm{OPT}\setminus R)-2(\beta+\Tilde{K})\epsilon-\Delta \nonumber\\
    &\geq f(R)+(1-\frac{1}{e})f'(\mathrm{OPT}\setminus R)-2(\beta+\Tilde{K})\epsilon-f(R)/3\nonumber\\
    &\geq (1-\frac{1}{3})f(R)+(1-\frac{1}{e})f'(\mathrm{OPT}\setminus R)-2(\beta+\Tilde{K})\epsilon \nonumber\\
    &\geq (1-\frac{1}{e})\left[f'(\mathrm{OPT}\setminus R) + f(R)\right]-2(\beta+\Tilde{K})\epsilon \tag{$e\leq 3$}\\
    &=(1-\frac{1}{e})f(\mathrm{OPT})-2(\beta+\Tilde{K})\epsilon. \tag{definition of $f'$}
\end{align}
The output of the algorithm is not necessarily $G$ because the values of the evaluated triplets are based on surrogate function $\hat{f}$. Denote $\mathcal{O}$ as the output of the algorithm and denote $G'$ as the best evaluated set (with respect to $\hat{f}$) with size $\ell+2$ (same as $G$). We must have that $\hat{f}(G')\geq \hat{f}(G)$. Also denote the final set (until violating budget) continuing $G'$ as $G''$. We have,
\begin{align}
    f(\mathcal{O}) &\geq \hat{f}(\mathcal{O}) -\epsilon \nonumber\\
    &\geq \hat{f}(G'') - \epsilon \tag{selection rule of the algorithm}\\
    &\geq f(G'') - 2\epsilon \nonumber\\
    &\geq f(G') - 2\epsilon \tag{$G' \subseteq G''$ and monotonicity of $f$}\\
    &\geq \hat{f}(G') - 3\epsilon \nonumber\\
    &\geq \hat{f}(G) - 3\epsilon \nonumber\\
    &\geq f(G) - 4\epsilon \nonumber\\
    &\geq (1-\frac{1}{e})f(\mathrm{OPT})-(4+2\beta+2\Tilde{K})\epsilon, \nonumber
\end{align}
finishing the proof.
\end{proof}

\subsection{Robustness of Offline Methods for Unconstrained Non-monotone Submodular Maximization}
\label{sec:non-monotone}
We consider the original randomized algorithm \textsc{RandomizedUSM} proposed in \citet{Buchbinder2012ATL}, which gives a $\frac{1}{2}$-approximation guarantee for unconstrained non-monotone submodular maximization. To restate \cref{thm:RandomUSM:robust} in the main paper, \textsc{RandomizedUSM} is a $(\frac{1}{2},\frac{5}{2}n)$-robust approximation algorithm for unconstrained non-monotone submodular maximization problem. The result follows directly from Corollary 2 of \citet{Fourati2023Randomized}.

\section{Comments on Lower bounds of Submodular CMAB}\label{apdx_lb_sub}
For the setting we explore in this paper, with stochastic (or even adversarial) knapsack-constrained combinatorial MAB with submodular expected rewards and just bandit feedback, it remains an open question if $\tilde{\mathcal{O}}(T^{1/2})$ expected cumulative $\alpha$-regret is possible (ignoring $n$ and $\beta$).
Both \cite{streeter2008online} and \cite{niazadeh2021online} analyze lower bounds for the adversarial setting.  However, \cite{streeter2008online} obtain bounds for $1$-regret (it is NP-hard in offline setting to obtain an approximation ratio better than $1-1/e$).  \cite{niazadeh2021online} obtain $\tilde{\Omega}(T^{2/3})$ lower bounds for the harder setting where feedback is only available during ``exploration'' rounds chosen by the agent, who incurs an associated penalty.

\section{Implementation of Algorithm OG$^\text{o}$} \label{implimentation:ogo}

\begin{algorithm}[h]
    \caption{Online Greedy for Opaque Feedback Model (OG$^\text{o}$)}
    \label{alg:OGO}
\begin{algorithmic}
   \STATE {\bfseries Input:} set of base arms $\Omega$, horizon $T$, cost for each arm $c(a)$, budget $B$
   \STATE Initialize $n\gets|\Omega|$, $\cmin \gets \min_{a\in \Omega}\{c(a)\}$, $\beta \gets \frac{B}{\cmin}$,  $\gamma \gets n^{1/3}\beta\left(\frac{\log(n)}{T}\right)^{1/3}$, $\epsilon \gets \sqrt{\frac{\beta\log(n)}{\gamma T}}$
   \STATE Initialize $\boldsymbol{\omega}_1 \gets \text{ones}(\beta, n)$
   \FOR{$t \in [1,\cdots,T]$}
   \STATE $S_t \gets \emptyset$
   \STATE $l \gets \text{zeros}(\beta, n)$ \quad // \textit{ loss }
   \STATE Randomly sample a value $\xi \sim \text{Uniform}([0,1])$ 
   \IF{$\xi \leq \gamma$} %
   \STATE $e \sim \text{Uniform}(\{1,\cdots,\beta\})$ 
   \FOR{$i \in [1,\cdots, e-1]$} 
       \STATE \qquad // \textit{For experts before $e$, exploit}
       \STATE Select an arm $a$ with probability $\frac{\boldsymbol{\omega}_t[i,a]}{\sum \boldsymbol{\omega}_t[i,:]}$, re-sample if $a \in S_t$
       \STATE $S_t \gets S_t \cup \{a\}$ with probability $\frac{\cmin}{c(a)}$; $S_t \gets S_{t-1}$ otherwise
   \ENDFOR
   \STATE $a \sim \text{Uniform}(\{1,\cdots,n\}\backslash S_t)$ \quad // \textit{For expert $e$, explore}
   \STATE $S_t \gets S_t \cup \{a\}$ 
   \STATE Play action $S_t$, observe $f_t(S_t)$
   \STATE Update $l[i,j] \gets \frac{\cmin f_t(S_t)}{c(a)}$ for all $i = e$ and $j \neq a$ \quad // \textit{Feed  $\frac{\cmin f_t(S_t)}{c(a)}$ back to expert $e$ associated with action $a$}
   \STATE Update $\boldsymbol{\omega}_{t+1}[i,j] \gets \boldsymbol{\omega}_{t}[i,j]\exp(-\epsilon l[i,j])$ for all pairs of $i$ and $j$  
   \ELSE 
   \STATE \qquad  // \textit{Exploitation with probability $1-\gamma$}
   \FOR{$i \in [1,\cdots, \beta]$} 
       \STATE \qquad  // \textit{For experts before $e$, exploit}
       \STATE Select arm $a$ with probability $\frac{\boldsymbol{\omega}_t[i,a]}{\sum \boldsymbol{\omega}_t[i,:]}$, re-sample if $a \in S_t$
       \STATE $S_t \gets S_t \cup \{a\}$ with probability $\frac{\cmin}{c(a)}$; $S_t \gets S_{t-1}$ otherwise
   \ENDFOR
   \STATE Play action $S_t$, observe $f_t(S_t)$
   \STATE $\boldsymbol{\omega}_{t+1}[i,j] \gets \boldsymbol{\omega}_{t}[i,j]$
   \quad // \textit{Since feeding back 0 to all expert-action payoffs, loss is 0, no update}
   \ENDIF 
   \ENDFOR
\end{algorithmic}
\end{algorithm}

In this section we describe implementation details and parameter selection for OG$^\text{o}$ algorithm \cite{streeter2008online}. The choice of exploration probability is given by the original paper:$\gamma = n^{1/3}\beta\left(\frac{\log(n)}{T}\right)^{1/3}$, where $\beta=B/\cmin$. Note that in the original paper, $B$ is used instead of $\beta$, because they assume the minimum cost is 1. Here we generalize it to arbitrary non-negative costs. $\epsilon$ is the learning rate for Randomized Weighted Majority (WMR) expert algorithm \cite{arora2012multiplicative}. It is chosen by setting the derivative of regret upper bound to zero, which is $\epsilon = \sqrt{\frac{\log(n)}{T_e}}$, where $T_e$ is the time spent on updating expert $e$. Since it explores with probability $\gamma$, and there are $\beta$ expert algorithms, we have $T_e \approx \frac{\gamma T}{\beta}$. Thus we pick $\epsilon = \sqrt{\frac{\beta\log(n)}{\gamma T}}$. In experiments, there are many cases the chosen $\gamma$ is large or even larger than 1, so we cap the probability of exploring $\gamma$ by 1/2 to avoid exploring too much. Note that unlike a hard budget in our setting, for OG$^\text{o}$, it only requires the budget to be satisfied in expectation, so in general we might choose sets over budget. \cref{alg:OGO} is the pseudo code for implementation details of OG$^\text{o}$.

\section{Dealing with Small Time Horizons in Experiments} \label{apdx:smallT}

In \cref{sec:appl-CETC:submod}, we used $N=\Tilde{K}n$ as an upper bound on the number of function evaluations for both C-ETC-K and C-ETC-Y, where $n$ is the number of base arms and $\Tilde{K}$ is an upper bound of the cardinality of any feasible sets. When the time horizon $T$ is small, it is possible that the exploration phase will not finish due to the  formula being optimized for $m$ (the number of plays for each action queried by $\mathcal{A}$) uses a loose bound on the exploitation time.   When this is the case, we select the largest $m$ (closest to the formula) for which we can guarantee that exploration will finish.  Recall that for \textsc{C-ETC-Y} and \textsc{C-ETC-K}, the number of oracle calls can only be upper bounded in advance. 

We first calculate $\moptrnd$ using \eqref{eq:int:value:m}:
\begin{align}
    \moptrnd= \ceil*{\frac{\delta^{2/3}T^{2/3}\log(T)^{1/3}}{2\Tilde{K}^{2/3}n^{2/3}}}. \nonumber
\end{align}
Note that a (slightly tighter) upper bound on the number of subsets evaluated during the exploration phase (with $\tilde{K}$ bounding the number of iterations of the greedy process) is
\begin{align}
    N &\leq n+(n-1)+\cdots+(n-\Tilde{K}+1) \nonumber\\
    &= \left(n-\frac{\Tilde{K}}{2}+\frac{1}{2}\right)\Tilde{K}.  \nonumber
\end{align}
We compare $\left(n-\frac{\Tilde{K}}{2}+\frac{1}{2}\right)\Tilde{K}\moptrnd$ with $T$. 

\begin{itemize}
    \item \textbf{Case 1.} If $\left(n-\frac{\Tilde{K}}{2}+\frac{1}{2}\right)\Tilde{K}\moptrnd < T$, C-ETC can finish exploring. We select $m=\moptrnd$. 
    \item \textbf{Case 2.} If $\left(n-\frac{\Tilde{K}}  {2}+\frac{1}{2}\right)\Tilde{K}\moptrnd \geq T$, it is possible that the algorithm cannot finish exploring.  %
    In this case, we will find a new $m$, so that the exploration can be guaranteed to finish.   
    We select the largest $m$ (closest to $m^\dagger$) so that the exploration time is upper bounded by $T$, 
    \begin{align}
        m=\frac{T}{\left(n-\frac{\Tilde{K}}{2}+\frac{1}{2}\right)\Tilde{K}}. \nonumber
    \end{align}
\end{itemize}

\section{Other Related Work for Adversarial CMAB with Knapsack constraints}\label{apdx:rad}
   \citet{streeter2008online}  propose and analyze an algorithm for adversarial CMAB with submodular rewards,  full-bandit feedback, and under a knapsack constraint (though only in expectation, taken over randomness in the algorithm). We also use this as a baseline in our experiments in \cref{sec:exp}.
The authors  adapted a simpler greedy algorithm than the one we adapt \citep{khuller1999budgeted}, using an $\epsilon$-greedy exploration type framework.  We provide evidence in our experiments %
that their algorithm requires large horizons to learn.  %
The offline algorithm they adapted %
achieves an approximation ratio $(1-1/e)$ %
for budgets that exactly match the cost used up by the greedy solution, but otherwise does not achieve a constant approximation \citep{khuller1999budgeted}.

In \citep{golovin2014online}, the authors propose an algorithm for adversarial setting with submodular rewards %
when there is a  matroid constraint (neither knapsack nor matroid constraints are special cases of the other).

\section{Related work on Stochastic Submodular CMAB with Semi-Bandit Feedback}\label{apdx_semi}
There are also a number of works that require additional ``semi-bandit'' feedback.  For combinatorial MAB with submodular rewards, a common type of semi-bandit feedback are marginal gains \citep{lin2015stochastic,yue2011linear, yu2016linear, takemori2020submodular}, which enable the learner to take actions of maximal cardinality or budget, receive a corresponding reward, and gain information not just on the set but individual elements.  For the full-bandit setting we consider, to greedily build a solution, we need to spend time taking small cardinality actions to estimate their quality, incurring regret.

\end{document}